\documentclass[letterpaper]{article} % DO NOT CHANGE THIS
\usepackage{aaai23}  % DO NOT CHANGE THIS
\usepackage{times}  % DO NOT CHANGE THIS
\usepackage{helvet}  % DO NOT CHANGE THIS
\usepackage{courier}  % DO NOT CHANGE THIS
\usepackage[hyphens]{url}  % DO NOT CHANGE THIS
\usepackage{graphicx} % DO NOT CHANGE THIS
\usepackage{natbib}  % DO NOT CHANGE THIS AND DO NOT ADD ANY OPTIONS TO IT
\usepackage{caption} % DO NOT CHANGE THIS AND DO NOT ADD ANY OPTIONS TO IT
\usepackage{algorithm}
\usepackage{algorithmic}
\usepackage[toc,page]{appendix}
\usepackage{newfloat}
\usepackage{listings}
\DeclareCaptionStyle{ruled}{labelfont=normalfont,labelsep=colon,strut=off} % DO NOT CHANGE THIS
\floatstyle{ruled}
\newfloat{listing}{tb}{lst}{}
\floatname{listing}{Listing}
\usepackage{amsfonts}
\usepackage{footmisc}

\usepackage[utf8]{inputenc}
\usepackage{xcolor}
\usepackage{booktabs}
\usepackage{amsthm}
\usepackage{listings}
\usepackage{inconsolata}
\usepackage{tikz}
\usepackage{pgfplots}
\usepackage{amsmath}
\usepackage{microtype}
\usepackage{subcaption}
\definecolor{pixel 0}{HTML}{FFFFFF}
\definecolor{pixel 1}{HTML}{FF0000} % red

\newcommand{\name}{\textsc{Disco}}
\newcommand{\popper}{\textsc{Popper}}
\newcommand{\metagol}{\textsc{Metagol}}
\newcommand{\ilasp}{\textsc{ILASP3}}
\newcommand{\ale}{\textsc{Aleph}}

\theoremstyle{definition}
\newtheorem{definition}{Definition}

\newtheorem{proposition}{Proposition}
\newtheorem{lemma}{Lemma}

\lstnewenvironment{myalgorithm}[1][] %defines the algorithm listing environment
{
    \lstset{ %this is the stype
        showspaces=false,               % show spaces adding particular underscores
        showstringspaces=false,         % underline spaces within strings
        mathescape=true,
        numbers=left,
        escapeinside={*}{*},
        columns=flexible,
        numbersep=2pt,        % default 10pt
        keywordstyle=\bfseries,
        keywords={,and, return, not, def, in, if, else, for, foreach, while, }
        numbers=left,
        xleftmargin=.0\textwidth,
        #1 % this is to add specific settings to an usage of this environment (for instnce, the caption and referable label)
    }
}
{}

\title{Learning Logic Programs By Discovering Where Not to Search}
\author {
    Andrew Cropper and Céline Hocquette
}
\affiliations {    
    University of Oxford\\
 andrew.cropper@cs.ox.ac.uk, celine.hocquette@cs.ox.ac.uk
}

\begin{document}

\maketitle

\begin{abstract}
    The goal of inductive logic programming (ILP) is to search for a hypothesis that generalises training examples and background knowledge (BK).
    To improve performance, we introduce an approach that, before searching for a hypothesis, first discovers \emph{where not to search}.
    We use given BK to discover constraints on hypotheses, such as that a number cannot be both even and odd.
    We use the constraints to bootstrap a constraint-driven ILP system.
    Our experiments on multiple domains (including program synthesis and game playing) show that our approach can (i) substantially reduce learning times by up to 97\%, and (ii) scale to domains with millions of facts.
\end{abstract}
\section{Introduction}

The goal of inductive logic programming (ILP) \cite{mugg:ilp} is to search for a hypothesis (a set of rules) that generalises training examples and background knowledge (BK), where hypotheses, examples, and BK are all logic programs.

To illustrate ILP, consider learning list transformation rules with an arbitrary head literal \emph{h}.
Assume we can build rules using the unary relations \emph{odd} and \emph{even} and the binary relations \emph{head} and \emph{tail}.
Then the rule space (the set of all possible rules) contains rules such as:

\begin{center}
\begin{tabular}{l}
\emph{r$_1$  = h $\leftarrow$ tail(A,A)}\\
\emph{r$_2$  = h $\leftarrow$ tail(A,B), tail(B,A)}\\
\emph{r$_3$  = h $\leftarrow$ tail(A,B), tail(B,C), tail(A,C)}\\
\emph{r$_4$  = h $\leftarrow$ tail(A,A), head(A,B), odd(B)}\\
\emph{r$_5$  = h $\leftarrow$ head(A,B), odd(B), even(B)}
\end{tabular}
\end{center}

\noindent
The hypothesis space (the set of all hypotheses) is the powerset of the rule space, so can be enormous.

To improve performance, users can impose an inductive bias \cite{tm:book} to restrict the hypothesis space\footnote{
All machine learning approaches need an inductive bias, i.e. bias-free learning is futile \cite{tm:book}.
}.
For instance, if told that \emph{tail} is irreflexive, some systems \cite{ilasp} will remove rules with the literal \emph{tail(A,A)} from the rule space, such as $r_1$ and $r_4$.
As removing a rule removes all hypotheses that contain it, a strong bias can greatly reduce the hypothesis space.

The main limitation with existing approaches is that they need a human to provide a strong bias, e.g. they need to be told that some relations are irreflexive.
Furthermore, existing bias approaches, such as mode declarations \cite{progol}, cannot describe many useful properties, such as antitransitivity and functional dependencies \cite{DBLP:journals/dke/MannilaR94}.
In general, developing automatic bias discovery approaches is a grand challenge in the field \cite{ilpintro}.

To overcome this limitation, we introduce an automated bias discovery approach. 
The key idea is to use given BK to discover how to restrict the hypothesis space \emph{before} searching for a solution\footnote{
A \emph{solution} is a hypothesis that generalises the examples.
An \emph{optimal} solution is the smallest solution in terms of its size in literals.
}.
For instance, consider the previous list transformation example.
% For instance, consider the previous list transformation example.
Assume we have BK with only the facts:

\begin{center}
\begin{tabular}{lll}
\emph{head(ijcai,i)} & \emph{tail(ijcai,jcai)} & \emph{even(2)}\\
\emph{head(ecai,e)} & \emph{tail(ecai,cai)} & \emph{even(4)}\\
\emph{head(cai,c)} & \emph{tail(jcai,cai)} & \emph{odd(1)}\\
\emph{tail(ai,i)} & \emph{tail(cai,ai)} & \emph{odd(3)}\\
\end{tabular}
\end{center}

\noindent
Given this BK, if we adopt a \emph{closed world assumption} \cite{cwa} we can deduce that some rules will be unsatisfiable \emph{regardless of the concept we want to learn}, i.e. regardless of specific training examples.
For instance, as there is no fact of the form \emph{tail(A,A)}, we can deduce that \emph{tail} is irreflexive, and thus remove $r_1$ and $r_4$ from the rule space as their bodies are unsatisfiable\footnote{
These properties may only hold with the given BK.
However, as the ILP problem is defined in terms of the given BK, our approach is optimally sound (Proposition \ref{prop:sound}).
}.
Similarly, we can deduce that \emph{tail} is asymmetric and antitransitive and that \emph{odd} and \emph{even} are mutually exclusive and thus remove rules $r_2$, $r_3$, and $r_5$.
% from the rule space.
With this bias discovery approach, we have substantially reduced the hypothesis space before searching for a solution, i.e. we have discovered \emph{where not to search}.

Our bias discovery approach works in two stages.
First, we use the given BK to discover functional dependencies and relational properties, such as irreflexivity, asymmetry, and antitransitivity.
To do so, we use a bottom-up approach \cite{savnik1993bottom} implemented in answer set programming (ASP) \cite{asp}.
% \cite{asp}.
Second, we use the properties to build constraints to restrict the hypothesis space.
For instance, if we discover that \emph{even} and \emph{odd} are mutually exclusive, we build constraints to prohibit rules with both the body literals \emph{odd(A)} and \emph{even(A)}.
We use these constraints to bootstrap a constraint-driven ILP system \cite{popper}. \
The constraints remove non-optimal hypotheses from the hypothesis space so that the system never considers them when searching for a solution.
% and thus reduces learning time.

\paragraph{Novelty, Impact, and Contributions.}
The novelty of this paper is the idea of \emph{automatically deducing constraints about the hypothesis space before searching the hypothesis space}.
As far as we are aware, this idea has not been explored before.
The impact is vastly improved learning performance, demonstrated on a diverse set of tasks and domains.
For instance, our approach can reduce learning times by up to 97\%.
Moreover, as the idea connects many AI fields, including  program synthesis, constraint programming, and knowledge representation, there is much potential for broad research to build on this idea.

Overall, we make the following contributions:

\begin{itemize}
\item We introduce the constraint discovery problem and define \emph{optimally sound} constraints.

\item We describe a bias discovery approach that automatically discovers functional dependencies and relational properties, such as asymmetry and antitransitivity.
We prove that our approach is optimally sound.

\item We implement our approach in ASP and use it to bootstrap a constraint-driven ILP system. 

\item We experimentally show on multiple domains that our approach can (i) substantially reduce learning times by up to 97\%, and (ii) scale to BK with millions of facts.
\end{itemize}
\section{Related Work}

\textbf{Program synthesis.}
The goal of program synthesis is to automatically generate computer programs from examples.
This topic, which \citet{gulwani2017program} consider the holy grail of AI, interests a broad community \cite{dilp,ellis:scc}.
Although our bias discovery idea could be applied to any form of program synthesis, we focus on ILP because it induces human-readable relational programs, often from small numbers of training examples \cite{ilpintro}.
Moreover, the logical representation naturally supports declarative knowledge in the form of logical constraints.

\textbf{ILP.}
Many systems allow a human to manually specify conditions for when a rule cannot be in a hypothesis \cite{progol,aleph,ilasp}.
Most systems only reason about the conditions \emph{after} constructing a hypothesis, such as Aleph's rule pruning mechanism.
By contrast, we automatically discover constraints and remove rules that violate them from the hypothesis space \emph{before} searching for a hypothesis.

\textbf{Constraints.}
Many systems use constraints to restrict the hypothesis space \cite{aspal,inoue:mla,atom,hexmil,popper}.
For instance, the Apperception \cite{apperception} engine has several built-in constraints, such as a \emph{unity condition}, which requires that objects are connected via chains of binary relations.
By contrast, we automatically discover constraints before searching for a hypothesis.

\textbf{Bottom clauses.}
Many systems use mode declarations to build bottom clauses \cite{progol} to bound the hypothesis space.
Bottom clauses can be seen as informing an ILP system where to search.
Our approach is similar, as it restricts the hypothesis space.
%There are, however, many differences.
However, bottom clauses are example specific. To find a rule to cover an example, a learner constructs the bottom clause for that specific example, which it uses to bias the search.
By contrast, our bias discovery approach is task independent and only uses the BK, not the training examples.
Because of this difference, we can reuse any discovered bias across examples and tasks.
For instance, if we discover that the successor relation (\emph{succ}) is asymmetric, we can reuse this bias across multiple tasks.
In addition, because of our two-stage approach, we can amortise the cost of discovering BK constraints across tasks.
%Moreover, 

\textbf{Bias discovery.}
% Our approach is a form of bias discovery.
% \ac{mention bottom clauses}
\citet{modelearning} automatically deduce mode declarations from the BK, such as types and whether arguments should be ground.
Our approach is different because, as we use constraints, we can reason about properties that modes cannot, such as antitransitivity, functional dependencies, and mutual exclusivity. \citet{bridewell2007} learn structural constraints over the hypothesis space in a multi-task setting. 
By contrast, we discover biases before solving any task.
% of relations.

\textbf{Constraint induction.}
Inducing constraints is popular in AI \cite{DBLP:conf/aaai/RaedtPT18}.
In ILP, inducing constraints has been widely studied, notably by clausal discovery approaches \cite{claudien}.
These approaches induce constraints to include in a hypothesis to eliminate models.
% and require training examples.
By contrast, we do not include constraints in hypotheses.
Instead, we discover constraints to prune the hypothesis space.

\textbf{Preprocessing.}
% We formulate the problem of discovering constraints as a \emph{preprocessing} problem.
Our discovery approach is a form of \emph{preprocessing}, which has been widely studied in AI, notably to reduce the size of a SAT instance \cite{DBLP:conf/sat/EenB05}.
Other preprocessing approaches in ILP focus on reducing the size of BK \cite{alps} or predicate invention \cite{celine:bottom}.
By contrast, we discover constraints in the BK to prune the hypothesis space.

\textbf{Other work.}
% We use techniques similar to bottom-up approaches for discovering data dependencies \cite{DBLP:journals/dke/MannilaR94}.
Our approach is related to automated constraint generation in constraint programming \cite{colton_constraints}, finding unsatisfiable cores in SAT \cite{unsatcores}, and condensed representations in frequent pattern mining \cite{DBLP:conf/kr/RaedtR04}.
% For instance, \citet{colton_constraints} automatically detect implied constraints, i.e. constraints that are logically implied by the problem formulation but can also help in solving the problem more efficiently.

% there are two things here.

% 1) finding the constraints, where a claudien like system can be used,
% 2) using these during refinement, or search, which is what is described under semantic refinement in our KR 2004 paper for relational frequent pattern mining, but which should work in general,
% It is also mentioned in my book.  

\section{Problem Setting}
\label{sec:setting}
% We now describe our problem setting. 
We formulate our approach in the ILP learning from entailment setting \cite{luc:book}.
% , a constraint-driven approach to ILP.
% We first describe the ILP problem and then the preprocessing problem. 
We assume familiarity with logic programming \cite{lloyd:book} and ASP \cite{asp}.
The only clarification is that by \emph{constraint} we mean a Horn clause without a positive literal. 
% \ch{The constraints we consider are hypotheses constraints ie constraints about hypotheses.}
% , also known as a \emph{denial}.

\subsection{ILP Problem}
We define an ILP input.
% as used throughout the rest of this paper.
We restrict hypotheses and BK to definite programs.

\begin{definition}[\textbf{ILP input}]
\label{def:probin}
An ILP input is a tuple $(E^+, E^-, B, \mathcal{H})$ where $E^+$ and $E^-$ are sets of facts denoting positive and negative examples respectively, $B$ is BK, and $\mathcal{H}$ is a hypothesis space, i.e a set of possible hypotheses.
\end{definition}

% \ac{clarify that we assume these are given by the user}
% \ac{Datalog}

\noindent
We define an ILP solution:
% \footnote{
% % These definitions assume noiseless examples.
% ILP can handle noisy examples but we focus on the noiseless setting as our preprocessing approach only considers BK.
% % , not examples.
% }:

\begin{definition}[\textbf{ILP solution}]
\label{def:solution}
Given an ILP input $(E^+, E^-, B, \mathcal{H})$, a hypothesis $H \in \mathcal{H}$ is a \emph{solution} when it is \emph{complete} ($\forall e \in E^+, \; B \cup H \models e$) and \emph{consistent} ($\forall e \in E^-, \; B \cup H \not\models e$).
\end{definition}

% A hypothesis is a \emph{failure} when it is not a solution.
% A hypothesis is \emph{incomplete} when $\exists e \in E^+, \; H \cup B \not \models e$.
% A hypothesis is \emph{inconsistent} when $\exists e \in E^-, \; H \cup B \models e$. \ch{incomplete and inconsistent are not used in the framework, shall we remove this?}
% A hypothesis is \emph{totally incomplete} when $\forall e \in E^+, \; H \cup B \not \models e$.
\noindent
Let \emph{cost} : $\mathcal{H} \mapsto \mathbb{R}$ be an arbitrary function that measures the cost of a hypothesis.
We define an \emph{optimal} solution:

\begin{definition}[\textbf{Optimal solution}]
\label{def:opthyp}
Given an ILP input $(E^+, E^-, B, \mathcal{H})$, a hypothesis $H \in \mathcal{H}$ is \emph{optimal} when (i) $H$ is a solution, and (ii) $\forall H' \in \mathcal{H}$, where $H'$ is a solution, \emph{cost}($H$) $\leq$ \emph{cost}($H'$).
\end{definition}

\noindent
% In this paper, our cost function \emph{cost(H)} is the number of literals in the hypothesis $H$.
In this paper, our cost function is the number of literals in a hypothesis.
In addition, we use the notion of a \emph{task} to refer to the problem of finding an optimal solution for an ILP input.
% \ac{S5.1: The notion of "task" is not defined in S3.}

% $H$.

\subsection{Constraint Discovery Problem}
% A LFF learner continually searches for hypotheses and tests them on examples and BK to accumulate hypothesis constraints.
% We use constraints to restrict the hypothesis space.
% We use the term \emph{hypothesis constraint} to distinguish between the constraints we are interested in from the standard use of the term.
% A \emph{hypothesis constraint} is a constraint about hypotheses. 
% We denote as $\mathcal{C}$ the set of all possible constraints.
We denote a set of possible constraints as $\mathcal{C}$.
A hypothesis $H \in \mathcal{H}$ is \emph{consistent} with $C \subseteq \mathcal{C}$ if it does not violate any constraint in $C$.
We denote the subset of $\mathcal{H}$ consistent with $C$ as $\mathcal{H}_{C}$.
% Our goal is, therefore, to find a set of constraints that restrict the hypothesis space.
We define the \emph{constraint discovery} input:

\begin{definition}[\textbf{Constraint discovery input}]
A constraint discovery input is a tuple $(E^+, E^-, B, \mathcal{H}, \mathcal{C})$ where  $(E^+, E^-, B, \mathcal{H})$ is an ILP input and  $\mathcal{C}$ is a set of possible constraints.
\end{definition}

\noindent
We define the \emph{constraint discovery} problem:

\begin{definition}[\textbf{Constraint discovery problem}]
\label{def:preprob}
Given a constraint discovery input $(E^+, E^-, B, \mathcal{H}, \mathcal{C})$, the \emph{constraint discovery} problem is to find $C \subseteq \mathcal{C}$ such that $|\mathcal{H}_{C}| < |\mathcal{H}|$.
\end{definition}

\noindent
One might assume we want to discover \emph{sound} constraints:
\begin{definition}[\textbf{Sound constraints}]
Let $I = (E^+, E^-, B, \mathcal{H}, \mathcal{C})$ be a constraint discovery input.
Then $C \subseteq \mathcal{C}$ is \emph{sound} if and only if $\forall H \in \mathcal{H}$ if $H$ is a solution for $I$ then $H \in \mathcal{H}_{C}$.
\end{definition}

\noindent
However, we often want to eliminate non-optimal solutions from the hypothesis space. 
For instance, consider learning to recognise lists with a single element and the hypothesis:
\begin{center}
\begin{tabular}{l}
\emph{f(A) $\leftarrow$ length(A,B), one(B), two(B)}\\
\emph{f(A) $\leftarrow$ length(A,B), one(B)}
\end{tabular}
\end{center}
This hypothesis is a solution but is not optimal.
We would prefer to learn an optimal solution, such as:
\begin{center}
\begin{tabular}{l}
\emph{f(A) $\leftarrow$ length(A,B), one(B)}
\end{tabular}
\end{center}

\noindent
We, therefore, define \emph{optimally sound} constraints:

\begin{definition}[\textbf{Optimally sound constraints}]
Let $I = (E^+, E^-, B, \mathcal{H}, \mathcal{C})$ be a constraint discovery input. Then $C \subseteq \mathcal{C}$ is \emph{optimally sound} if and only if $\forall H \in \mathcal{H}$ if $H$ is an optimal solution for $I$ then $H \in \mathcal{H}_{C}$.
\end{definition}
\noindent
% This constraint discovery problem is general.
In the next section we present an approach that discovers optimally sound constraints using the BK.

\section{BK Constraint Discovery}
\label{sec:impl}

Our approach works in two stages. 
First, we use BK to identify relational properties and functional dependencies.
Second, we use the properties to build constraints on hypotheses to bootstrap an ILP system.

\subsection{Properties}
\label{sec:cons}
% \subsubsection{Properties and Dependencies}
\begin{table*}
\centering
% \small
% \footnotesize
\begin{tabular}{@{}llll@{}}
% \toprule
\textbf{Name} & \textbf{Property} & \textbf{Constraint} & \textbf{Example}\\
\midrule

Irreflexive & 
\emph{$\neg$p(A,A)} & 
\emph{$\leftarrow$ p(A,A)} & 
\emph{$\leftarrow$ brother(A,A)}
\\
Antitransitive & 
\emph{p(A,B), p(B,C) $\rightarrow$ $\neg$p(A,C)} & 
\emph{$\leftarrow$ p(A,B), p(B,C), p(A,C)} & 
\emph{$\leftarrow$ succ(A,B), succ(B,C), succ(A,C)}
\\
Antitriangular & 
\emph{p(A,B), p(B,C) $\rightarrow$ $\neg$p(C,A)} & 
\emph{$\leftarrow$ p(A,B), p(B,C), p(C,A)}&
\emph{$\leftarrow$ tail(A,B), tail(B,C), tail(C,A)}
\\
Injective & 
\emph{p(A,B), p(C,B) $\rightarrow$ A=C} & 
\emph{$\leftarrow$ p(A,B), p(C,B), A$\neq$C} & 
\emph{$\leftarrow$ succ(A,B), succ(C,B), A$\neq$C}
\\
Functional & 
\emph{p(A,B), p(A,C) $\rightarrow$ B=C} & 
\emph{$\leftarrow$ p(A,B), p(A,C), B$\neq$C} &
\emph{$\leftarrow$ length(A,B), length(A,C), B$\neq$C}
\\
Asymmetric & 
\emph{p(A,B) $\rightarrow$ $\neg$p(B,A)} & 
\emph{$\leftarrow$ p(A,B), p(B,A)} &
\emph{$\leftarrow$ mother(A,B), mother(B,A)}
\\
% Asymmetric$_{3,1}$ & \emph{p(A,B,C) $\rightarrow$ $\neg$p(A,C,B)}& \emph{$\leftarrow$ p(A,B,C), p(A,C,B)}\\ % ?
% Asymmetric$_{3,2}$ & \emph{p(A,B,C) $\rightarrow$ $\neg$p(B,A,C)} &\emph{$\leftarrow$ p(A,B,C), p(B,A,C)}\\ % ?
% Asymmetric$_{3,3}$ & \emph{p(A,B,C) $\rightarrow$ $\neg$p(B,C,A)} &\emph{$\leftarrow$ p(A,B,C), p(B,C,A)}\\ % ?
% Asymmetric$_{3,4}$ & \emph{p(A,B,C) $\rightarrow$ $\neg$p(C,A,B)} & \emph{$\leftarrow$ p(A,B,C), p(C,A,B)}\\ %
% Asymmetric$_{3,5}$ & \emph{p(A,B,C) $\rightarrow$ $\neg$p(C,B,A)} & \emph{$\leftarrow$ p(A,B,C), p(C,B,A)}\\ % append
Exclusive & 
\emph{p(A) $\rightarrow$ $\neg$q(A)} &
\emph{$\leftarrow$ p(A), q(A)} &
\emph{$\leftarrow$ odd(A), even(A)}
% Exclusive$_{2}$ & \emph{p(A,B) $\rightarrow$ $\neg$q(A,B)} &\emph{$\leftarrow$ p(A,B), q(A,B)}\\
% \bottomrule
\end{tabular}
\caption{
Properties and constraints. 
We generalise the properties, except antitransitive and antitriangular, to higher arities.
The relation \emph{succ/2} is the successor relation for natural numbers, such as \emph{succ(1,2)}, \emph{succ(2,3)}, \emph{succ(3,4)}, etc.
% (up to the maximum arity in the BK).
% Uppercase variables are all universally quantified.
% For instance, our notation for the \emph{antitransitive} property is shorthand for \emph{$\forall$ A,B,C $\leftarrow$ p(A,B), p(B,C), p(A,C)}.
}
\label{tab:props}
\end{table*}

Table \ref{tab:props} shows the properties we consider.
We generalise the properties, except antitransitive and antitriangular, to higher arities.
For instance, if a ternary relation \emph{p} is in the BK, we consider a ternary irreflexive constraint \emph{$\leftarrow$ p(A,A,A)}.
% We also generalise the binary functional dependencies (injective and functional) to higher arities.
Similarly, we also identify higher-arity functional dependencies \cite{DBLP:journals/dke/MannilaR94}.
For instance, for the relation \emph{append(Head,Tail,List)} we can determine that the third argument is functionally dependent on the first two.
The appendix describes the properties we consider and their generalisations to arity 3.%and their higher-arity generalisations.

\subsubsection{Property Identification}
\label{sec:discovery}
% , so we can remove them from the hypothesis space.
Rather than requiring a user to specify which properties in Table \ref{tab:props} hold for BK relations, we automatically discover this information.
% Therefore, our task is: \emph{given} (i) BK in the form of a logic program without function symbols, and (ii) a library of properties, \emph{determine} which properties hold for the relations in the BK.
There are many efficient algorithms for discovering data dependencies \cite{DBLP:journals/pvldb/PapenbrockEMNRZ15}.
However, as far as we are aware, no single algorithm can capture all the properties in Table \ref{tab:props}.
% Therefore, we describe an approach to discover which properties hold for the BK relations.
% \footnote{
%     Our implementation uses more efficient encodings than those shown, but we present simpler encodings for brevity and clarity.
%     % to make the paper easier to read and to avoid needing to introduce ASP syntax for aggregates.
% }.
% To be clear: our discovery approach is almost certainly less efficient compared to existing ones, and using more efficient algorithms is future work.
We, therefore, implement a bottom-up approach \cite{savnik1993bottom} in ASP.
The idea is to try to find a counter-example for each property.
For instance, for a binary relation \emph{p} to be irreflexive there cannot be a counter-example \emph{p(a,a)}.
To implement this idea, we \emph{encapsulate} all the relations in the BK, restricted to a user-specified set that may appear in a hypothesis. 
Specifically, for each relation \emph{p} with arity $a$ we add this rule to the BK:

\begin{center}
\begin{tabular}{l}
\emph{holds}$(p,(X_1,X_2,\dots,X_a)) \leftarrow p(X_1,X_2,\dots,X_a)$ 
\end{tabular}
\end{center}
\noindent
We then deduce properties with ASP programs.
For instance, we deduce asymmetry for binary relations by finding an answer set (a stable model) of the program:

\begin{center}
\begin{tabular}{l}
\emph{asymmetric(P) $\leftarrow$ holds(P,(\_,\_)), not non\_asymmetric(P)}\\
\emph{non\_asymmetric(P) $\leftarrow$ holds(P,(A,B)), holds(P,(B,A))}
\end{tabular}
\end{center}

% \noindent
% We deduce antitransitvity with the program:

% \begin{center}
% \begin{tabular}{l}
% \emph{antitransitive(P) $\leftarrow$ holds(P,(\_,\_)), not non\_antitransitive(P)}\\
% \emph{non\_antitransitive(P) $\leftarrow$ holds(P,(A,B)), holds(P,(B,C)), holds(P,(A,C))}
% \end{tabular}
% \end{center}

% \paragraph{Functional.}
% We deduce functionality with the program:

% \begin{center}
% \begin{tabular}{l}
% \emph{functional(P) $\leftarrow$ holds(P,(\_,\_)), not non\_functional(P)}\\
% \emph{non\_functional(P) $\leftarrow$ holds(P,(A,B)), holds(P,(A,C)), B != C}
% \end{tabular}
% \end{center}
\noindent
% \textbf{Exclusivity.}
Likewise, we deduce that two relations $P$ and $Q$ are mutually exclusive with the program:

\begin{center}
\begin{tabular}{l}
\emph{exclusive(P,Q) $\leftarrow$ holds(P,\_), holds(Q,\_), not both\_hold(P,Q)}\\
\emph{both\_hold(P,Q) $\leftarrow$ holds(P,Args), holds(Q,Args)}
\end{tabular}
\end{center}

\noindent
We deduce that a binary relation $P$ is functional with the program:

\begin{center}
\begin{tabular}{l}
\emph{functional(P) $\leftarrow$ holds(P,(\_,\_)), not non\_functional(P)}\\
\emph{non\_functional(P) $\leftarrow$ holds(P,(A,B)), holds(P,(A,C)), B!=C}
\end{tabular}
\end{center}

\noindent
The appendix includes all the ASP programs we consider.

% A limitation of this bottom-up approach is that we assume Datalog BK.
% We will address this limitation in future work.
% \begin{lstlisting}[frame=none]
% mutually_exclusive(P,Q):-
%     holds(P,_), holds(Q,_), not both_hold(P,Q).
% both_hold(P,Q):-
%     holds(P,A), holds(Q,A).
% \end{lstlisting}
% :-
%     unsat_pair(P,Q),
%     body_literal(C,P,1,(V1,)),
%     body_literal(C,Q,1,(V1,)).

% \subsubsection{Discovery vs Search tradeoff}
% \ac{I feel we need something here to say we need to tradeoff preprocessing vs the search.
% What I mean is, a reader might wonder whether time spent preprocessing would be better spent searching.
% I will write more here later}

\subsection{Constraints}

% \paragraph{Building constraints} 
The output of stage one is a set of properties that hold for background relations. If a property holds for a relation, we generate the corresponding constraint to prohibit hypotheses that violate the constraint.
Although these constraints can potentially be used by any ILP system, we implement our approach to work with \popper{} \cite{popper}.
% , especially those that frame the ILP problem as a constraint satisfaction problem \cite{hexmil,ilasp}.
% \ac{However, you may be able to explain it away. You need to explain why you chose Popper instead of another system in Section 4.3}
\popper{} is a natural choice because it frames the ILP problem as a constraint satisfaction problem.
% to implement our approach.
Moreover, it learns recursive programs, supports predicate invention, and is open-source\footnote{ILASP \cite{ilasp} is an alternative system but is closed-source and thus difficult to adapt.
HEXMIL \cite{hexmil} is also an alternative system but requires metarules (program templates) as input and is thus largely restricted to dyadic logic. \popper{}, by contrast, does not need metarules.
}.
We describe \popper{} and our modification named \name{}.

\subsubsection{\popper{}}
\popper{} takes as input BK, training examples, and a maximum hypothesis size and learns hypotheses as definite programs.
% (measured in terms of literals).
\popper{} starts with an ASP program $\mathcal{P}$ which can be viewed as a \emph{generator} program because each model (answer set) of $\mathcal{P}$ represents a hypothesis.
\popper{} uses a meta-language formed of head (\emph{h\_lit/3}) and body (\emph{b\_lit/3}) literals to represent hypotheses.
The first argument of each literal is the rule id, the second is the predicate symbol, and the third is the literal variables, where \emph{0} represents \emph{A}, \emph{1} represents \emph{B}, etc.
For instance, \popper{} represents the rule \emph{last(A,B) $\leftarrow$ tail(A,C), head(C,B)} as the set \emph{\{h\_lit(0,last,(0,1)), b\_lit(0,tail,(0,2)), b\_lit(0,head,(2,1))\}}.
A hypothesis constraint in \popper{} is a constraint written in its meta-language.
% Let $C$ be a set of hypotheses constraints. A set of definite clauses $H$ is consistent with $C$ if, when written in $\mathcal{L}$, H does not violate any constraint in $C$. 
For instance, the constraint \emph{$\leftarrow$ h\_lit(R,last,(0,1)), b\_lit(R,last,(1,0))} prunes rules that contain the head literal \emph{last(A,B)} and the body literal \emph{last(B,A)}.
% \end{tabular}
% \end{center}

% is violated by the rule $last(A,B) \leftarrow last(B,A)$.
% \noindent
% We now describe the LFF learner \popper{}, which we use in our experiments. A complete description is available in \cite{popper}.

% \noindent
\popper{} uses a generate, test, and constrain loop to search for a solution.
% for the ILP problem (Definition \ref{def:solution}).
In the generate stage, it uses an ASP solver to find a model of $\mathcal{P}$.
If there is a model, \popper{} converts it to a hypothesis and tests it on the examples; otherwise, it increments the hypothesis size and loops again.
If a hypothesis is not a solution, \popper{} builds hypothesis constraints and adds them to $\mathcal{P}$ to eliminate models and thus prunes the hypothesis space.
For instance, if a hypothesis does not entail all the positive examples, \popper{} builds a specialisation constraint to prune more specific hypotheses.
This loop repeats until \popper{} finds an optimal solution or there are no more hypotheses to test.

\subsubsection{\name{}}
We augment \popper{} with the ability to use the constraints from our discovery approach.
The input from the user is the same as for \popper{} except that we require the BK to be a Datalog program. 
In other words, facts and rules are allowed but not function symbols. 
We call this augmented version \name{}. 
We condition the constraints to only apply to a relation \emph{p} if a property holds for \emph{p}.
For instance, we add an asymmetric constraint to \name{}:

\begin{center}
\begin{tabular}{l}
\emph{$\leftarrow$ asymmetric(P), b\_lit(R,P,(A,B)), b\_lit(R,P,(B,A))}
\end{tabular}
\end{center}

\noindent
If \emph{asymmetric(mother)} holds, \name{} builds the constraint:

\begin{center}
\begin{tabular}{l}
\emph{$\leftarrow$ b\_lit(R,mother,(A,B)), b\_lit(R,mother,(B,A))}
\end{tabular}
\end{center}

\noindent
This constraint prunes all models that contain the literals \emph{b\_lit(R,mother,(A,B))} and \emph{b\_lit(R,mother,(B,A))}, i.e. all rules with the body literals \emph{mother(A,B)} and \emph{mother(B,A)}.
This constraint applies to all variable substitutions for \emph{A} and \emph{B} and all rules \emph{R}.
For instance, the constraint prunes the rule:

\begin{center}
\begin{tabular}{l}
\emph{h $\leftarrow$ sister(A,B), sister(B,C), mother(C,D), mother(D,C)}
\end{tabular}
\end{center}

% \noindent
% \textbf{Antitransitive.}
% Likewise, we add an antitransitive constraint to \name{}:

% \begin{center}
% \begin{tabular}{l}
% \emph{$\leftarrow$ antitransitive(P), b\_lit(R,P,(A,B)), b\_lit(R,P,(B,C)), b\_lit(R,P,(A,C))}
% \end{tabular}
% \end{center}

% \noindent
% For instance, if \emph{antitransitive(tail)} holds, \name{} builds the constraint:

% \begin{center}
% \begin{tabular}{l}
% \emph{$\leftarrow$ b\_lit(R,tail,(A,B)), b\_lit(R,tail,(B,C)), b\_lit(R,tail,(A,C))}
% \end{tabular}
% \end{center}

% \paragraph{Functional.}
% We add a functional constraint to \name{}:

% \begin{center}
% \begin{tabular}{l}
% \emph{$\leftarrow$ functional(P), b\_lit(R,P,(A,B)), b\_lit(R,P,(A,C)), B != C}
% \end{tabular}
% \end{center}

% \noindent
% For instance, if \emph{functional(length)} holds, \name{} builds the constraint:

% \begin{center}
% \begin{tabular}{l}
% \emph{$\leftarrow$ b\_lit(R,length,(A,B)), b\_lit(R,length,(A,C)), B != C}
% \end{tabular}
% \end{center}

\noindent
% \textbf{Exclusivity.}
Likewise, we add an exclusivity constraint to \name{}:
% \ch{shall it be for any arity as in the example in 4.3?}

\begin{center}
\begin{tabular}{l}
\emph{$\leftarrow$ exclusive(P,Q), b\_lit(R,P,Vars), b\_lit(R,Q,Vars)}
\end{tabular}
\end{center}

\noindent
For instance, if \emph{exclusive(odd,even)} holds, \name{} builds the constraint:

\begin{center}
\begin{tabular}{l}
\emph{$\leftarrow$ b\_lit(R,odd,Vars), b\_lit(R,even,Vars)}
\end{tabular}
\end{center}

\noindent
We add a functional constraint to \name{}:
\begin{center}
\begin{tabular}{l}
\emph{$\leftarrow$ functional(P), b\_lit(R,P,(A,B)), b\_lit(R,P,(A,C)), C!=B}
\end{tabular}
\end{center}

\noindent
For instance, if \emph{functional(tail)} holds, \name{} builds the constraint:

\begin{center}
\begin{tabular}{l}
\emph{$\leftarrow$ b\_lit(R,tail,(A,B)), b\_lit(R,tail,(A,C)), C!=B}
\end{tabular}
\end{center}

\noindent
The ASP encodings for all the constraints are in the appendix.

To avoid complications with recursion, we do not use head predicate symbols (those in the examples) when discovering properties from the BK.

% The appendix includes all the ASP encodings for the constraints.

% \ac{Perhaps cover all constraints from the table in 4.4? I’m not sure about this as I can see how every constraints is mapped to an ASP expression, but it might look more serious to a reviewers if it is complete.}

\subsection{Optimal Soundness}
We now prove that our approach only builds optimally sound constraints, i.e. it will not remove optimal solutions from the hypothesis space.
We first show the following lemma:
\begin{lemma} \label{lemma}
Each property in Table \ref{tab:props} has an associated constraint with an unsatisfiable body.
\end{lemma}
\begin{proof}
Follows from rewriting each property and the universal quantification. 
% For instance, assume a relation $p$ is asymmetric. Then:
% \begin{align*}
% \forall A,B, p(A,B) \rightarrow \neg p(B,A).\\
% \iff \forall A,B, \neg p(A,B) \lor \neg p(B,A)\\
% \iff \forall A,B, \neg (p(A,B) \land p(B,A))\\
% \end{align*}
% Therefore the body of the constraint \emph{$\leftarrow$ p(A,B), p(B,A)} is unsatisfiable.
\end{proof}
\noindent
We show the main result:
\begin{proposition}[\textbf{Optimally sound constraint discovery}] \label{prop:sound}
Given the properties in Table \ref{tab:props}, our approach builds optimally sound constraints.
\end{proposition}
\begin{proof}
Let $H \in \mathcal{H} \setminus  \mathcal{H}_{C}$. 
Assume $H$ is an optimal solution. 
Since $H \in \mathcal{H}$ but $H \not\in \mathcal{H}_{C}$ there must be a hypothesis constraint $C_1 \in C$ such that $H$ violates $C_1$. $C_1$ is a constraint from Table 1 and prunes rules. Then there exists a rule $C_2 \in H$ and a substitution $\theta$ such that $C_1 \theta \subset C_2$. 
$C_1$ has been built from our library of properties and thus has an unsatisfiable body according to Lemma \ref{lemma}.
Since $C_1$ has an unsatisfiable body, then the body of $C_2$ is unsatisfiable. Thus $C_2$ does not change the coverage of $H$. Then $H \setminus C_2$ is a solution which contradicts our assumption.
\end{proof}

\section{Experiments}
\label{sec:exp}
To evaluate our claim that BK constraint discovery can reduce learning times, our experiments aim to answer the question:

\begin{description}
\item[Q1] Can BK constraint discovery reduce learning times?
\end{description}

\noindent
To answer \textbf{Q1}, we compare the performance of \popper{}\footnote{We use Popper 2.0.0 \cite{popper2}.} and \name{} (\popper{} with BK constraint discovery).

To understand how much our approach can improve learning performance, our experiments aim to answer the question:

\begin{description}
\item[Q2] What effect does BK constraint discovery have on learning times given larger hypothesis spaces?
\end{description}

\noindent
To answer \textbf{Q2}, we compare the performance of \popper{} and \name{} on progressively larger hypothesis spaces.

To understand the scalability of our approach, our experiments aim to answer the question:

\begin{description}
    \item[Q3] How long does our BK constraint discovery approach take given larger BK?
\end{description}

\noindent
To answer \textbf{Q3}, we measure BK constraint discovery time on progressively larger BK.

As our approach is novel, there is no state-of-the-art to compare against, i.e. comparing \name{} against other systems will not allow us to evaluate the benefits of BK constraint discovery.
We have, however, included a comparison of \name{} with other systems in the appendix, which shows that \name{} comprehensively outperforms state-of-the-art systems.

% \ac{add a reference to the comparrison in the appendix?}

\subsection{Experimental Domains}
We use six domains.
We briefly describe them.
The appendix contains more details and example solutions.

\textbf{Michalski trains.}
The goal is to find a hypothesis that distinguishes eastbound and westbound trains \cite{michalski:trains}.
We use four increasingly complex tasks.
% Figure \ref{fig:trains-prog} shows a solution for the \emph{trains2} task. 

% \begin{figure}[t]
% \begin{center}
% \begin{tabular}{l}
% \emph{east(A) $\leftarrow$ car(A,C), roof\_open(C), load(C,B), triangle(B)}\\
% \emph{east(A) $\leftarrow$ car(A,C), car(A,B), roof\_closed(B),}\\
% \emph{$\quad\quad\quad\quad\;$ two\_wheels(C), roof\_open(C)}\\
% \end{tabular}
% \end{center}
% \caption{\centering
% Example solution for the \emph{trains2} task.
% }
% \label{fig:trains-prog}
% \end{figure}

\textbf{IMDB.}
This real-world dataset \cite{mihalkova2007} contains relations between movies, actors, directors, gender and movie genre. 
We learn the binary relations \emph{workedunder}, a more complex variant \emph{workedwithsamegender}, and the disjunction of the two.

\textbf{Chess.} 
The task is to learn a rule for the king-rook-king (\emph{krk}) endgame where the white king protects its rook \cite{celine:bottom}.

\textbf{Zendo.} 
Zendo is a multi-player game in which players try to identify a secret rule by building structures.
We use four increasingly complex tasks.
%The appendix includes example Zendo solutions.

\textbf{IGGP.}
The goal of \emph{inductive general game playing} \cite{iggp} (IGGP) is to induce rules to explain game traces from the general game playing competition \cite{ggp}.
We use six games: \emph{minimal decay (md)}, \emph{rock-paper-scissors (rps)}, \emph{buttons}, \emph{attrition}, \emph{centipede}, and \emph{coins}.
% Figure \ref{fig:iggp-prog} shows a solution for the the \emph{minimal decay} (\emph{md}) task.

% \begin{figure}[t]
% \begin{center}
% \begin{tabular}{l}
% \emph{next\_val(A,5) $\leftarrow$ does(A,player,press\_button)}\\
% \emph{next\_val(A,B) $\leftarrow$ does(A,player,noop)t, true\_val(A,C), succ(B,C)}
% \end{tabular}
% \end{center}
% \caption{
% \centering
% Example solution for the IGGP \emph{md} task.
% }
% \label{fig:iggp-prog}
% \end{figure}

\textbf{Program synthesis.}
% Learning recursive programs is a major challenge \cite{ilp20} and most systems cannot learn recursive programs.
We use a standard synthesis dataset \cite{popper}\footnote{
Our constraint discovery implementation requires Datalog BK, a common restriction \cite{hexmil,apperception}.
However, the BK for the synthesis tasks is a definite program.
% It is infeasible to transform the dataset to a Datalog program, as it requires a fact 
Therefore, to discover BK constraints, we use a Datalog subset of the BK restricted to an alphabet with 10 symbols (0-9), where the BK constraint discovery time is 4s.
We use the definite program BK for the learning task.
% \ch{This is not clear, why not include bias discovery time for the subset of the BK?}
}.
% \ac{need to mention to awkward issue with the synthesis tasks}.

\subsection{Experimental Setup}
% The key details are as follows.
We enforce a timeout of 20 minutes per task.
We measure the mean and standard error over 10 trials.
We round times over one second to the nearest second.
The appendix includes all the experimental details and example solutions.

\textbf{Q1.}
We compare the performance of \popper{} and \name{} on all tasks.
We measure predictive accuracy and learning time.
We separately measure BK constraint discovery time.
% \footnote{something about the overhead}

\textbf{Q2.}
We compare the performance of \popper{} and \name{} when varying the size of the hypothesis space.
We vary the maximum size of a rule allowed in a hypothesis, i.e. the maximum number of literals allowed in a rule.
We use the IGGP \emph{md} task to answer this question.

\textbf{Q3.}
We measure BK constraint discovery time on progressively larger BK.
We generate BK for the synthesis tasks.
The BK facts are relations between strings of a finite alphabet.
For instance, the BK contains facts such as:

\begin{center}
\begin{tabular}{ll}
\emph{string((1,3,3,7))} & \emph{head((1,3,3,7),(1,))}\\
\emph{tail((1,3,3,7),(3,3,7))} & \emph{append((1,),(3,3,7),(1,3,3,7))}\\
\end{tabular}
\end{center}

\noindent
We generate larger BK by increasing the size of the alphabet.

\subsection{Experimental Results}

\subsubsection{Q1}
Table \ref{tab:qtimes} shows the learning times. 
It shows that on these datasets \name{} (i) never needs more time than \popper{}, and (ii) can drastically reduce learning time.
A paired t-test confirms the significance of the difference at the $p < 0.01$ level.
For instance, for the \emph{buttons} task (the appendix includes an example solution), the learning time is reduced from 686s to 25s, a \textbf{96\%} reduction.

Table \ref{tab:preprocesstime} shows that BK constraint discovery time is always less than a second, except for the synthesis tasks.
For instance, for the real-world \emph{imdb3} task, BK constraint discovery takes 0.02s yet reduces learning time from 366s to 287s, a 21\% reduction.

\begin{table}[ht]
\centering
\begin{tabular}{@{}l|ccc@{}}
    \textbf{Task} & \textbf{\popper{}} & \textbf{\name{}} & \textbf{Change}\\
    \midrule
\emph{trains1} & 5 $\pm$ 0.1 & 4 $\pm$ 0.1 & \textbf{-20\%} \\
\emph{trains2} & 5 $\pm$ 0.2 & 4 $\pm$ 0.3 & \textbf{-20\%} \\
\emph{trains3} & 27 $\pm$ 0.8 & 22 $\pm$ 0.6 & \textbf{-18\%} \\
\emph{trains4} & 24 $\pm$ 0.8 & 20 $\pm$ 0.5 & \textbf{-16\%} \\
\midrule
\emph{zendo1} & 8 $\pm$ 2 & 6 $\pm$ 1 & \textbf{-25\%} \\
\emph{zendo2} & 32 $\pm$ 2 & 31 $\pm$ 2 & \textbf{-3\%} \\
\emph{zendo3} & 33 $\pm$ 2 & 31 $\pm$ 1 & \textbf{-6\%} \\
\emph{zendo4} & 24 $\pm$ 3 & 24 $\pm$ 3 & 0\% \\
\midrule
\emph{imdb1} & 1 $\pm$ 0 & 1 $\pm$ 0 & 0\% \\
\emph{imdb2} & 2 $\pm$ 0.1 & 2 $\pm$ 0 & 0\% \\
\emph{imdb3} & 366 $\pm$ 23 & 287 $\pm$ 17 & \textbf{-21\%} \\
\midrule
\emph{krk} & 48 $\pm$ 6 & 9 $\pm$ 0.6 & \textbf{-81\%} \\
\midrule
\emph{rps} & 37 $\pm$ 1 & 6 $\pm$ 0.2 & \textbf{-83\%} \\
\emph{centipede} & 47 $\pm$ 2 & 9 $\pm$ 0.2 & \textbf{-80\%} \\
\emph{md} & 142 $\pm$ 7 & 13 $\pm$ 0.4 & \textbf{-90\%} \\
\emph{buttons} & 686 $\pm$ 109 & 25 $\pm$ 1 & \textbf{-96\%} \\
\emph{attrition} & 410 $\pm$ 20 & 57 $\pm$ 2 & \textbf{-86\%} \\
\emph{coins} & 496 $\pm$ 19 & 345 $\pm$ 18 & \textbf{-30\%} \\
\emph{buttons-goal} & 11 $\pm$ 0.2 & 5 $\pm$ 0.1 & \textbf{-54\%} \\
\emph{coins-goal} & 122 $\pm$ 6 & 76 $\pm$ 2 & \textbf{-37\%} \\
\midrule
\emph{dropk} & 4 $\pm$ 0.3 & 3 $\pm$ 0.2 & \textbf{-25\%} \\
\emph{droplast} & 41 $\pm$ 3 & 23 $\pm$ 2 & \textbf{-43\%} \\
\emph{evens} & 33 $\pm$ 7 & 9 $\pm$ 1 & \textbf{-72\%} \\
\emph{finddup} & 51 $\pm$ 8 & 32 $\pm$ 4 & \textbf{-37\%} \\
\emph{last} & 4 $\pm$ 0.4 & 3 $\pm$ 0.2 & \textbf{-25\%} \\
\emph{len} & 31 $\pm$ 5 & 16 $\pm$ 2 & \textbf{-48\%} \\
\emph{sorted} & 74 $\pm$ 5 & 23 $\pm$ 1 & \textbf{-68\%} \\
\emph{sumlist} & 554 $\pm$ 122 & 320 $\pm$ 40 & \textbf{-42\%} \\

\end{tabular}
\caption{
Learning times in seconds.
We round times over one second to the nearest second.
Error is standard error.
% The learning times include BK constraint discovery time,
% except the synthesis ones, as explained in Footnote \ref{footsyn}.
}
\label{tab:qtimes}
\end{table}

\begin{table}[ht]
\centering
% \footnotesize
% \small
% \begin{tabular}{l|c}
\begin{tabular}{@{}l|c@{}}
    \textbf{Domain} & \textbf{Time}\\ 
    \midrule
\emph{trains} & 0.22 $\pm$ 0.00 \\
\midrule
\emph{zendo} & 0.03 $\pm$ 0.00 \\
\midrule
\emph{imdb} & 0.02 $\pm$ 0.00 \\
\midrule
\emph{krk} & 0.10 $\pm$ 0.00 \\
\midrule
\emph{rps} & 0.02 $\pm$ 0.00 \\
\emph{centipede} & 0.02 $\pm$ 0.00 \\
\emph{md} & 0.01 $\pm$ 0.00 \\
\emph{buttons} & 0.02 $\pm$ 0.00 \\
\emph{attrition} & 0.01 $\pm$ 0.00 \\
\emph{coins} & 0.03 $\pm$ 0.00 \\
\midrule
\emph{synthesis} & 4.00 $\pm$ 0.40\\
\end{tabular}
\caption{
BK constraint discovery times in seconds.
% We round times over one second to the nearest second.
% The learning times include BK constraint discovery time,
% except the synthesis ones, as explained in Footnote \ref{footsyn}.
% Error is standard error.
}
\label{tab:preprocesstime}
\end{table}
% The reduction in learning time is because \name{} searches a subset of the hypothesis space, which is evident when comparing the number of hypotheses considered, as shown in Table \ref{tab:q1}\footnote{
% \name{} sometimes considers more hypotheses than \popper{}, such as in the \emph{imdb1} and \emph{imdb2} tasks.
% The reason is that although \name{} searches a subset of the hypothesis space, there is no guarantee that it will find a solution before \popper{}, i.e. the systems might consider hypotheses in a different order.
% }.    
% For instance, to find a solution for the \emph{md} task, \popper{} considers 4564 hypotheses whereas \name{} only considers 866.

To understand why our approach works, consider the \emph{rps} task.
Our approach quickly (0.02s) discovers that the relation \emph{succ} is irreflexive, injective, functional, antitransitive, antitriangular, and asymmetric.
The resulting constraints reduce the number of rules in the hypothesis space from 1,189,916 to 70,270. 
This reduction in the number of rules in turn considerably reduces the number of programs to consider. 
As shown in Table \ref{tab:num_progs}, the number of programs generated and tested is reduced from 6297 to 988, an \textbf{84\%} reduction.
% This reduction in turn reduces learning time from 10 minutes (timeout) \ch{rather 20min?} to 116s.
% \ac{@AC, rephrase as this section is pants}
% , i.e. only 6\% of the rules. \ch{these numbers are not easily comparable with the ones in the table. Maybe it would help to clarify what is the number of programs generated (ie a subset of the hypothesis space) for readers unfamiliar with popper and refer to 4.4?}

Table \ref{tab:qaccs} shows the predictive accuracies. 
It shows that \name{} (i) has equal or higher predictive accuracy than \popper{} on all the tasks, and (ii) can improve predictive accuracy.
% , such as for the \emph{sorted} task.
A McNemar's test confirms the significance of the difference at the p $<$ 0.01 level.

\begin{table}[t]
\centering
% \footnotesize
% \small
\begin{tabular}{@{}l|ccc@{}}
    \textbf{Task} & \textbf{\popper{}} & \textbf{\name{}} & \textbf{Change}\\
    \midrule
\emph{trains1} & 617 $\pm$ 11 & 575 $\pm$ 14 & \textbf{-6\%} \\
\emph{trains2} & 617 $\pm$ 10 & 556 $\pm$ 34 & \textbf{-9\%} \\
\emph{trains3} & 2532 $\pm$ 2 & 2341 $\pm$ 2 & \textbf{-7\%} \\
\emph{trains4} & 2712 $\pm$ 0 & 2519 $\pm$ 0 & \textbf{-7\%} \\
\midrule
\emph{zendo1} & 2179 $\pm$ 783 & 2011 $\pm$ 698 & \textbf{-7\%} \\
\emph{zendo2} & 6972 $\pm$ 414 & 6437 $\pm$ 475 & \textbf{-7\%} \\
\emph{zendo3} & 7828 $\pm$ 493 & 7377 $\pm$ 444 & \textbf{-5\%} \\
\emph{zendo4} & 5512 $\pm$ 738 & 5303 $\pm$ 628 & \textbf{-3\%} \\
\midrule
\emph{imdb1} & 5 $\pm$ 0 & 7 $\pm$ 0 & +40\% \\
\emph{imdb2} & 34 $\pm$ 1 & 39 $\pm$ 1 & +14\% \\
\emph{imdb3} & 330 $\pm$ 0.3 & 300 $\pm$ 0.7 & \textbf{-9\%} \\
\midrule
\emph{krk} & 502 $\pm$ 59 & 56 $\pm$ 7 & \textbf{-88\%} \\
\midrule
\emph{rps} & 6297 $\pm$ 7 & 988 $\pm$ 2 & \textbf{-84\%} \\
\emph{centipede} & 2312 $\pm$ 0 & 947 $\pm$ 0 & \textbf{-59\%} \\
\emph{md} & 2415 $\pm$ 51 & 714 $\pm$ 10 & \textbf{-70\%} \\
\emph{buttons} & 4610 $\pm$ 57 & 1248 $\pm$ 9 & \textbf{-72\%} \\
\emph{attrition} & 25560 $\pm$ 188 & 7221 $\pm$ 67 & \textbf{-71\%} \\
\emph{coins} & 63370 $\pm$ 1778 & 45037 $\pm$ 1357 & \textbf{-28\%} \\ 
\emph{buttons-goal} & 109570 $\pm$ 169 & 49555 $\pm$ 65 & \textbf{-54\%} \\
\emph{coins-goal} & 23533 $\pm$ 0 & 18483 $\pm$ 0 & \textbf{-21\%} \\
\midrule
\emph{dropk} & 535 $\pm$ 34 & 433 $\pm$ 25 & \textbf{-19\%} \\
\emph{droplast} & 420 $\pm$ 17 & 330 $\pm$ 15 & \textbf{-21\%} \\
\emph{evens} & 877 $\pm$ 81 & 415 $\pm$ 39 & \textbf{-52\%} \\
\emph{finddup} & 7335 $\pm$ 919 & 4887 $\pm$ 641 & \textbf{-33\%} \\
\emph{last} & 560 $\pm$ 101 & 310 $\pm$ 56 & \textbf{-44\%} \\
\emph{len} & 1940 $\pm$ 239 & 1390 $\pm$ 148 & \textbf{-28\%} \\
\emph{sorted} & 2630 $\pm$ 159 & 1311 $\pm$ 105 & \textbf{-50\%} \\
\emph{sumlist} & 9422 $\pm$ 3121 & 5360 $\pm$ 2061 & \textbf{-43\%} \\

\end{tabular}
\caption{
Number of programs generated.
Error is standard error. 
}
\label{tab:num_progs}
\end{table}

\begin{table}[t]
\centering
% \footnotesize
% \small
\begin{tabular}{@{}l|ccc@{}}
    \textbf{Task} & \textbf{\popper{}} & \textbf{\name{}} & \textbf{Change}\\
    \midrule
\emph{trains1} & 100 $\pm$ 0 & 100 $\pm$ 0 & 0\% \\
\emph{trains2} & 98 $\pm$ 0 & 98 $\pm$ 0 & 0\% \\
\emph{trains3} & 99 $\pm$ 0 & 99 $\pm$ 0 & 0\% \\
\emph{trains4} & 100 $\pm$ 0 & 100 $\pm$ 0 & 0\% \\
\midrule
\emph{zendo1} & 99 $\pm$ 0 & 99 $\pm$ 0 & 0\% \\
\emph{zendo2} & 96 $\pm$ 1 & 97 $\pm$ 1 & \textbf{+1\%} \\
\emph{zendo3} & 93 $\pm$ 1 & 93 $\pm$ 2 & 0\% \\
\emph{zendo4} & 97 $\pm$ 0 & 97 $\pm$ 0 & 0\% \\
\midrule
\emph{imdb1} & 100 $\pm$ 0 & 100 $\pm$ 0 & 0\% \\
\emph{imdb2} & 100 $\pm$ 0 & 100 $\pm$ 0 & 0\% \\
\emph{imdb3} & 100 $\pm$ 0 & 100 $\pm$ 0 & 0\% \\
\midrule
\emph{krk} & 99 $\pm$ 0 & 99 $\pm$ 0 & 0\% \\
\midrule
\emph{rps} & 100 $\pm$ 0 & 100 $\pm$ 0 & 0\% \\
\emph{centipede} & 100 $\pm$ 0 & 100 $\pm$ 0 & 0\% \\
\emph{md} & 100 $\pm$ 0 & 100 $\pm$ 0 & 0\% \\
\emph{buttons} & 100 $\pm$ 0 & 100 $\pm$ 0 & 0\% \\
\emph{attrition} & 98 $\pm$ 0 & 98 $\pm$ 0 & 0\% \\
\emph{coins} & 100 $\pm$ 0 & 100 $\pm$ 0 & 0\% \\
\emph{buttons-goal} & 98 $\pm$ 1 & 99 $\pm$ 0 & \textbf{+1\%} \\
\emph{coins-goal} & 100 $\pm$ 0 & 100 $\pm$ 0 & 0\% \\
\midrule
\emph{dropk} & 100 $\pm$ 0 & 100 $\pm$ 0 & 0\% \\
\emph{droplast} & 100 $\pm$ 0 & 100 $\pm$ 0 & 0\% \\
\emph{evens} & 100 $\pm$ 0 & 100 $\pm$ 0 & 0\% \\
\emph{finddup} & 98 $\pm$ 0 & 99 $\pm$ 0 & \textbf{+1\%} \\
\emph{last} & 100 $\pm$ 0 & 100 $\pm$ 0 & 0\% \\
\emph{len} & 100 $\pm$ 0 & 100 $\pm$ 0 & 0\% \\
\emph{sorted} & 97 $\pm$ 2 & 97 $\pm$ 2 & 0\% \\
\emph{sumlist} & 90 $\pm$ 6 & 100 $\pm$ 0 & \textbf{+11\%} \\

\end{tabular}
\caption{
Predictive accuracies.
We round times over one second to the nearest second.
Error is standard error.
}
\label{tab:qaccs}
\end{table}

There are two reasons for this accuracy improvement.
First, \popper{} sometimes does not find a good solution within the time limit. 
By contrast, as there are fewer hypotheses for \name{} to consider (Table \ref{tab:num_progs}), it sometimes finds a solution quicker.
Second, as our approach is optimally sound (Proposition \ref{prop:sound}), it is guaranteed to lead to a hypothesis space that is a subset of the original one yet still contains all optimal solutions.
According to the Blumer bound \cite{blumer:bound}, given two hypotheses spaces of different sizes, searching the smaller space will result in higher predictive accuracy compared to searching the larger one, assuming the target hypothesis is in both. 

\subsubsection{Q2}
Table \ref{tab:q2-times} shows that \name{} can drastically reduce learning time as the hypothesis space grows (relative to \popper{}). 
% \ch{for any hypothesis space size tested? I want to avoid confusion with "the reduction is larger as the hypothesis space grows" which is not the case}
For instance, for the \emph{md} task with a maximum rule size of 6 the learning times of \popper{} and \name{} are 113s and 10s respectively.
With a maximum rule size of 8, \popper{} times out after 20 minutes, whereas \name{} learns a solution in 47s.

\begin{table}[ht!]
% \small
\centering
% \captionsetup{justification=centering}
\begin{tabular}{@{}c|cccc@{}}
% \toprule    
\textbf{Task} & \textbf{Size} & \textbf{\popper{}} & \textbf{\name{}} & \textbf{Change}\\
\midrule
\emph{md} & 5 & 12 $\pm$ 0.9 & 3 $\pm$ 0.3 & \textbf{-75\%} \\
\emph{md} & 6 & 113 $\pm$ 2 & 10 $\pm$ 0.1 & \textbf{-91\%} \\
\emph{md} & 7 & 864 $\pm$ 156 & 23 $\pm$ 0.9 & \textbf{-97\%} \\
\emph{md} & 8 & \emph{timeout} & 47 $\pm$ 2 & \textbf{-96\%} \\
\emph{md} & 9 & \emph{timeout} & 48 $\pm$ 3 & \textbf{-96\%} \\
\emph{md} & 10 & \emph{timeout} & 52 $\pm$ 0.1 & \textbf{-95\%} \\
% \bottomrule
\end{tabular}
\caption{
% Learning times on the \emph{trains2} and  \emph{md} tasks when progressively increasing the maximum rule size and thus the hypothesis space.
Learning times when progressively increasing the maximum rule size and thus the hypothesis space.
The timeout is 20 minutes (1200s).
We round times over one second to the nearest second.
Error is standard error. 
}
\label{tab:q2-times}
\end{table}

\subsubsection{Q3}
Figure \ref{fig:scalability} shows that our approach scales linearly in the size of the BK and can scale to millions of facts.
For instance, for BK with around 8m facts, our approach takes around 47s.

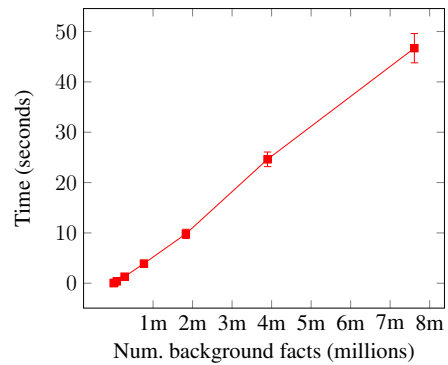
\begin{figure}[ht!]
\centering
\begin{tikzpicture}[scale=.7]
    \begin{axis}[
    scaled x ticks = false,
    xlabel=Num. background facts (millions),
    ylabel=Time (seconds),
    xtick={1000000,2000000,3000000,4000000,5000000,6000000,7000000,8000000},
    xticklabels={1m,2m,3m,4m,5m,6m,7m,8m},
    ylabel style={yshift=-1mm},
    label style={font=\Large},
    label style={font=\Large},
    tick label style={font=\Large},
    legend style={legend pos=north west,style={nodes={right}}}
    ]

% preprocessing
\addplot+[red,mark=square*,mark options={fill=red},error bars/.cd,y dir=both,y explicit]
table [
x=num_facts,
y=duration,
col sep=comma,
y error plus expr=\thisrow{error},
y error minus expr=\thisrow{error}
] {data/scaling-preprocess.csv};

% \legend{Clingo loading, Constraint discovery}
    \end{axis}
  \end{tikzpicture}
 \caption{
 BK constraint discovery time when increasing the number of background facts.
 }
  \label{fig:scalability}
\end{figure}

\section{Conclusions and Limitations}

To improve learning performance, we have introduced a bias discovery approach.
The three key ideas are (i) use the BK to discover a bias to restrict the hypothesis space, (ii) express the bias as constraints, and (iii) discover constraints \emph{before} searching for a solution.
Proposition \ref{prop:sound} shows that our approach is optimally sound.
% \ac{why does it work}
% \ac{q5: Are all motivating examples for D7 in the spirit of the given one, namely avoiding clauses with unsatisfiable bodies? Is it precisely your point that the constraint <- one(X), two(X) ought to be discovered from the BK?}
% \ac{problem decomposition}
Our experimental results on six domains show that our approach can (i) substantially reduce learning times, and (ii) scale to BK with millions of facts.

\subsection*{Limitations and Future Work}

\textbf{Finite BK.} 
Our constraint discovery approach is sufficiently general to handle definite programs as BK. 
However, as our implementation uses ASP, we require a finite grounding of the BK.
This restriction means that our implementation cannot handle BK with an infinite grounding, such as when reasoning about continuous values. 
Future work should address this limitation, such as by using top-down dependency discovery methods \cite{DBLP:journals/aicom/FlachS99}.

\textbf{CWA.}
We adopt a closed-world assumption to discover constraints from the given BK.
For instance, we assume that \emph{odd(2)} does not hold if not given as BK.
As almost all ILP systems adopt a CWA, this limitation only applies if our approach is used with a system that does not make the CWA. 
% \ac{FIX}
% We have revised the text to clarify this limitation.
We also assume that the BK is noiseless, i.e. if a fact is true in the BK, then it is meant to be true.
Handling noisy BK is an open challenge \cite{ilpintro} that is beyond the scope of this paper.

\textbf{Relational properties.}
We use a predefined set of relational properties and dependencies. The main direction for future work, therefore, is to discover more general properties and constraints.
For instance, consider the two rules \emph{h $\leftarrow$ empty(A), head(A,B)} and \emph{h $\leftarrow$ empty(A), tail(A,B)}. 
The bodies of these rules are unsatisfiable because an empty list cannot have a head or a tail.
We cannot, however, currently capture this information.
Therefore, we think that this paper raises two research challenges of (i) identifying more general properties, and (ii) developing approaches to efficiently discover properties.
\section*{Code, Data, and Appendices}
% A longer version of this paper with the appendices is available at \url{https://arxiv.org/pdf/2202.09806.pdf}.
The experimental code and data are available at \url{https://github.com/logic-and-learning-lab/aaai23-disco}.

\section*{Acknowledgements}
The first author is supported by the EPSRC fellowship \emph{The Automatic Computer Scientist} (EP/V040340/1).
The second author is supported by the EPSRC grant \emph{Explainable Drug Design}.
For the purpose of Open Access, the author has applied a CC BY public copyright licence to any Author Accepted Manuscript version arising from this submission.

\begin{appendices}

\section{Properties}
Table \ref{tab:props2} shows the relational properties we use, up to arity three.
% and generalisations to higher arities.
% For each property, we describe (i) the property, (ii) the ASP encoding to deduce the property, and (iii) the corresponding \name{} constraint. 

\begin{table*}
\centering
% \small
\footnotesize
\begin{tabular}{@{}llll@{}}
% \toprule
\textbf{Name} & \textbf{Property} & \textbf{Constraint} & \textbf{Example}\\
\midrule
Irreflexive$_{aaa}$ & 
\emph{$\neg$p(A,A,A)} & 
\emph{$\leftarrow$ p(A,A,A)} & 
\emph{$\leftarrow$ modulo(A,A,A)} 
\\

Injective$_{abc-dbc}$ & 
\emph{p(A,B,C), p(D,B,C) $\rightarrow$ A=D} & 
\emph{$\leftarrow$ p(A,B,C), p(D,B,C), A$\neq$D} & 
\emph{$\leftarrow$ add(A,B,C), add(D,B,C), A$\neq$D}
\\
Injective$_{abc-adc}$ & 
\emph{p(A,B,C), p(A,D,C) $\rightarrow$ B=D} & 
\emph{$\leftarrow$ p(A,B,C), p(A,D,C), B$\neq$D} & 
\emph{$\leftarrow$ add(A,B,C), add(A,D,C), B$\neq$D}
\\
Functional$_{abc-abd}$\footnote{Called unique\_ab\_c in the ASP encoding.} & 
\emph{p(A,B,C), p(A,B,D) $\rightarrow$ C=D} & 
\emph{$\leftarrow$ p(A,B,C), p(A,B,D), C$\neq$D} &
\emph{$\leftarrow$ add(A,B,C), add(A,B,D), C$\neq$D}
\\

% prop(unique_a_bc,P):- body_pred(P,3), not unique_a_bc_(P).
% prop(unique_ab_c,P):- body_pred(P,3), not unique_ab_c_(P).
% prop(unique_ac_b,P):- body_pred(P,3), not unique_ac_b_(P).
% prop(unique_b_ac,P):- body_pred(P,3), not unique_b_ac_(P).
% prop(unique_bc_a,P):- body_pred(P,3), not unique_bc_a_(P).
% prop(unique_c_ab,P):- body_pred(P,3), not unique_c_ab_(P).

Asymmetric$_{abc-acb}$ & 
\emph{p(A,B,C) $\rightarrow$ $\neg$p(A,C,B)} & 
\emph{$\leftarrow$ p(A,B,C), p(A,C,B)} & 
\emph{$\leftarrow$ cons(A,B,C), cons(A,C,B)}
\\
Asymmetric$_{abc-bac}$ & 
\emph{p(A,B,C) $\rightarrow$ $\neg$p(B,A,C)} & 
\emph{$\leftarrow$ p(A,B,C), p(B,A,C)} &
\emph{$\leftarrow$ cons(A,B,C), cons(B,A,C)}
\\
Asymmetric$_{abc-bca}$ & 
\emph{p(A,B,C) $\rightarrow$ $\neg$p(B,C,A)} & 
\emph{$\leftarrow$ p(A,B,C), p(B,C,A)} & 
\emph{$\leftarrow$ cons(A,B,C), cons(B,C,A)}
\\
Asymmetric$_{abc-cab}$ & 
\emph{p(A,B,C) $\rightarrow$ $\neg$p(C,A,B)} & 
\emph{$\leftarrow$ p(A,B,C), p(C,A,B)} & 
\emph{$\leftarrow$ cons(A,B,C), cons(C,A,B)}
\\
Asymmetric$_{abc-cba}$ & 
\emph{p(A,B,C) $\rightarrow$ $\neg$p(C,B,A)} & 
\emph{$\leftarrow$ p(A,B,C), p(C,B,A)} & 
\emph{$\leftarrow$ select(A,B,C), select(C,B,A)}
\\

Exclusive$_{ab}$ & 
\emph{p(A,B) $\rightarrow$ $\neg$q(A,B)} &
\emph{$\leftarrow$ p(A,B), q(A,B)} &
\emph{$\leftarrow$ head(A,B), tail(A,B)}
\\
Exclusive$_{abc}$ & 
\emph{p(A,B,C) $\rightarrow$ $\neg$q(A,B,C)} &
\emph{$\leftarrow$ p(A,B,C), q(A,B,C)} & \emph{$\leftarrow$ select(A,B,C), append(A,B,C)}
\\
Singleton & 
\emph{p(A), p(B) $\rightarrow$ A=B} &
\emph{$\leftarrow$ p(A), p(B), A$\neq$B} &
\emph{$\leftarrow$ one(A), one(B)}
\\
\end{tabular}
\caption{
Properties and constraints. 
This table is a supplement to Table 1.
% (up to the maximum arity in the BK).
% Uppercase variables are all universally quantified.
% For instance, our notation for the \emph{antitransitive} property is shorthand for \emph{$\forall$ A,B,C $\leftarrow$ p(A,B), p(B,C), p(A,C)}.
}
\label{tab:props2}
\end{table*}

\section{ASP Encoding}
Figure \ref{fig:asp} shows the ASP encoding.
In practice, we also use optional types to reduce grounding.
However, for brevity, we only show the untyped encodings.
We also only show the encodings for unary, binary, and ternary relations.

\section{\name{} Constraints}
Figure \ref{fig:popper} shows the constraints used by \name{}.

\begin{figure*}[ht!]
\begin{lstlisting}
body_pred(P,1):-holds(P,(_,)).
body_pred(P,2):-holds(P,(_,_)).
body_pred(P,3):-holds(P,(_,_,_)).
prop(antitransitive,P):- body_pred(P,2), not antitransitive_aux(P).
prop(antitriangular,P):- body_pred(P,2), not antitriangular_aux(P).
prop(asymmetric_ab_ba,P):- holds(P,(A,B)), not holds(P,(B,A)).
prop(asymmetric_abc_acb,P):- holds(P,(A,B,C)), not holds(P,(A,C,B)).
prop(asymmetric_abc_bac,P):- holds(P,(A,B,C)), not holds(P,(B,A,C)).
prop(asymmetric_abc_bca,P):- holds(P,(A,B,C)), not holds(P,(B,C,A)).
prop(asymmetric_abc_cab,P):- holds(P,(A,B,C)), not holds(P,(C,A,B)).
prop(asymmetric_abc_cba,P):- holds(P,(A,B,C)), not holds(P,(C,B,A)).
prop(singleton,P):- body_pred(P,_), #count{Vars : holds(P,Vars)} == 1.
prop(unsat_pair,P,Q):- body_pred(Q,A), body_pred(P,A), P > Q, #count{Vars : holds(P,Vars), holds(Q,Vars)} == 0.
prop(unique_a_b,P):- body_pred(P,2), not unique_a_b_(P).
prop(unique_b_a,P):- body_pred(P,2), not unique_b_a_(P).
prop(unique_a_bc,P):- body_pred(P,3), not unique_a_bc_(P).
prop(unique_ab_c,P):- body_pred(P,3), not unique_ab_c_(P).
prop(unique_ac_b,P):- body_pred(P,3), not unique_ac_b_(P).
prop(unique_b_ac,P):- body_pred(P,3), not unique_b_ac_(P).
prop(unique_bc_a,P):- body_pred(P,3), not unique_bc_a_(P).
prop(unique_c_ab,P):- body_pred(P,3), not unique_c_ab_(P).
antitransitive_aux(P):- holds(P,(A,B)), holds(P,(B,C)), holds(P,(A,C)).
antitriangular_aux(P):-  holds(P,(A,B)), holds(P,(B,C)),
unique_a_b_(P):-holds(P,(A,_)), #count{B : holds(P,(A,B))} > 1.
unique_b_a_(P):-holds(P,(_,B)), #count{A : holds(P,(A,B))} > 1.
unique_a_bc_(P):-holds(P,(A,_,_)), #count{B,C : holds(P,(A,B,C))} > 1.
unique_ab_c_(P):-holds(P,(A,B,_)), #count{C : holds(P,(A,B,C))} > 1.
unique_ac_b_(P):-holds(P,(A,_,C)), #count{B : holds(P,(A,B,C))} > 1.
unique_b_ac_(P):-holds(P,(_,B,_)), #count{A,C : holds(P,(A,B,C))} > 1.
unique_bc_a_(P):-holds(P,(_,B,C)), #count{A : holds(P,(A,B,C))} > 1.
unique_c_ab_(P):-holds(P,(_,_,C)), #count{A,B : holds(P,(A,B,C))} > 1.
\end{lstlisting}
\caption{ASP encoding to discover properties and functional dependencies in the BK.}
\label{fig:asp}
\end{figure*}

\begin{figure*}
\begin{lstlisting}
:- prop(asymmetric_ab_ba,P), body_literal(Rule,P,_,(A,B)), body_literal(Rule,P,_,(B,A)).
:- prop(asymmetric_abc_acb,P), body_literal(Rule,P,_,(A,B,C)), body_literal(Rule,P,_,(A,C,B)).
:- prop(asymmetric_abc_bac,P), body_literal(Rule,P,_,(A,B,C)), body_literal(Rule,P,_,(B,A,C)).
:- prop(asymmetric_abc_bca,P), body_literal(Rule,P,_,(A,B,C)), body_literal(Rule,P,_,(B,C,A)).
:- prop(asymmetric_abc_cab,P), body_literal(Rule,P,_,(A,B,C)), body_literal(Rule,P,_,(C,A,B)).
:- prop(asymmetric_abc_cba,P), body_literal(Rule,P,_,(A,B,C)), body_literal(Rule,P,_,(C,B,A)).
:- prop(unique_a_b,P), body_literal(Rule,P,_,(A,_)), #count{B : body_literal(Rule,P,_,(A,B))} > 1.
:- prop(unique_a_bc,P), body_literal(Rule,P,_,(A,_,_)), #count{B,C : body_literal(Rule,P,_,(A,B,C))} > 1.
:- prop(unique_ab_c,P), body_literal(Rule,P,_,(A,B,_)), #count{C : body_literal(Rule,P,_,(A,B,C))} > 1.
:- prop(unique_ac_b,P), body_literal(Rule,P,_,(A,_,C)), #count{B : body_literal(Rule,P,_,(A,B,C))} > 1.
:- prop(unique_b_a,P), body_literal(Rule,P,_,(_,B)), #count{A : body_literal(Rule,P,_,(A,B))} > 1.
:- prop(unique_b_ac,P), body_literal(Rule,P,_,(_,B,_)), #count{A,C : body_literal(Rule,P,_,(A,B,C))} > 1.
:- prop(unique_bc_a,P), body_literal(Rule,P,_,(_,B,C)), #count{A : body_literal(Rule,P,_,(A,B,C))} > 1.
:- prop(unique_c_ab,P), body_literal(Rule,P,_,(_,_,C)), #count{A,B : body_literal(Rule,P,_,(A,B,C))} > 1.
:- prop(antitransitive,P), body_literal(Rule,P,_,(A,B)), body_literal(Rule,P,_,(B,C)), body_literal(Rule,P,_,(A,C)).
:- prop(antitriangular,P), body_literal(Rule,P,_,(A,B)), body_literal(Rule,P,_,(B,C)), body_literal(Rule,P,_,(C,A)).
:- prop(singleton,P), body_literal(Rule,P,_,_), #count{Vars : body_literal(Rule,P,A,Vars)} > 1.
:- prop(unsat_pair,P,Q), body_literal(Rule,P,_,Vars), body_literal(Rule,Q,_,Vars).
\end{lstlisting}
\caption{ASP encoding of constraints in \name{}.}
\label{fig:popper}
\end{figure*}%\input{docs/C-soundness}

\section{Experiments}

\subsection{Experimental domains}
We describe characteristics of the domains and tasks used in our experiments in Table \ref{tab:dataset} and \ref{tab:tasks}. Figure \ref{fig:sols} shows example solutions for some of the tasks.

\begin{table}[ht!]
\footnotesize
\centering
\begin{tabular}{@{}l|cccc@{}}
% \toprule
\textbf{Task} & \textbf{\# examples} & \textbf{\# relations} & \textbf{\# constants} & \textbf{\# facts}\\
\midrule
\emph{trains} & 1000 & 20 & 8561 & 28503 \\
\midrule
% \emph{imdb1} & 383 / 0 & 5 & 299 & 869 \\
% \emph{imdb2} & 383 / 71442 & 5 & 299 & 869 \\
% \emph{imdb3} & 69418 / 2438 & 5 & 299 & 869 \\
\emph{imdb1} & 383 & 6 & 299 & 1330 \\
\emph{imdb2} & 71825 & 6 & 299 & 1330 \\
\emph{imdb3} & 121801 & 6 & 299 & 1330 \\
\midrule
\emph{zendo1} & 100 & 16 & 1049 & 2270\\
\emph{zendo2} & 100 & 16 & 1047 & 2184\\
\emph{zendo3} & 100 & 16 & 1100 & 2320\\
\emph{zendo4} & 100 & 16 & 987 & 2087\\
\midrule
\emph{md} & 54 & 12 & 13 & 29 \\
\emph{buttons} & 530 & 13 & 60 & 656 \\
\emph{rps} & 464 & 6 & 64 & 405 \\
\emph{coins} & 2544 & 9 & 110 & 1101 \\
\emph{centipede} & 26 & 34 & 61 & 138 \\
\emph{attrition} & 672 & 12 & 65 & 163 \\
\midrule
\emph{krk} & 50 & 8 & 162 & 6744 \\
\midrule
\emph{dropk} & 20 & 10 & n/a & n/a\\
\emph{droplast} & 20 & 10 & n/a & n/a\\
\emph{evens} & 20 & 10 & n/a & n/a\\
\emph{finddup} & 20 & 10 & n/a & n/a\\
\emph{last} & 20 & 10 & n/a & n/a\\
\emph{len} & 20 & 10 & n/a & n/a\\
\emph{sorted} & 20 & 10 & n/a & n/a\\
\emph{sumlist} & 20 & 10 & n/a & n/a\\
% \bottomrule
\end{tabular}
\caption{Experimental domain description.}

% \ac{make clear the complexity ...}
\label{tab:dataset}
\end{table}

\begin{table}[ht!]
\footnotesize
\centering
\begin{tabular}{@{}l|ccc@{}}
% \toprule
\textbf{Task} & \textbf{\# rules} & \textbf{\# literals} & \textbf{max literals per rule}\\
\midrule
\emph{train1} & 1 & 6 &  6\\
\emph{train2} & 2 & 11 & 6\\
\emph{train3} & 3 & 17 &  7\\
\emph{train4} & 4 & 26 &  7\\
\midrule
\emph{zendo1} & 1 & 7 & 7 \\
\emph{zendo2} & 2 & 14 & 7 \\
\emph{zendo3} & 3 & 20 & 7\\
\emph{zendo4} & 4 & 23 & 7\\
\midrule
\emph{imdb1} & 1 & 5 & 5 \\
\emph{imdb2} & 1 & 5 & 5 \\
\emph{imdb3} & 2 & 10 & 5 \\
\midrule
\emph{krk} & 1 & 8 & 8 \\
\midrule
\emph{md} & 2 & 11 & 6 \\
\emph{buttons} & 10 & 61 & 7\\
\emph{rps} & 4 & 25 & 7\\
\emph{coins} & 16 & 45 & 7\\
\emph{attrition} & 3 & 14 & 5\\
\emph{centipede} & 2 & 8 & 4\\
\midrule
\emph{dropk} & 2 & 7 & 4\\
\emph{droplast} & 2 & 8 & 5\\
\emph{evens} & 2 & 7 & 5\\
\emph{finddup} & 2 & 7 & 4\\
\emph{last} & 2 & 7 & 4\\
\emph{len} & 2 & 7 & 4\\
\emph{sorted} & 2 & 9 & 6\\
\emph{sumlist} & 2 & 7 & 5\\
% \bottomrule
\end{tabular}
\caption{
Experimental tasks description.
}

\label{tab:tasks}
\end{table}

\paragraph{Michalski trains.}
The goal of these tasks is to find a hypothesis that distinguishes eastbound and westbound trains \cite{michalski:trains}. There are four increasingly complex tasks. %Listings \ref{trains2} and \ref{trains4} show a solution for the \emph{trains2} and \emph{trains4} tasks. 
% , which says that a train is eastbound if (i) it has a long carriage with two wheels and another long carriage with three wheels, and (ii) ...
There are 1000 examples but the distribution of positive and negative examples is different for each task.
We randomly sample the examples and split them into 80/20 train/test partitions.

\paragraph{Zendo.} Zendo is an inductive game in which one player, the Master, creates a rule for structures made of pieces with varying attributes to follow. The other players, the Students, try to discover the rule by building and studying structures which are labelled by the Master as following or breaking the rule. The first student to correctly state the rule wins. We learn four increasingly complex rules for structures made of at most 5 pieces of varying color, size, orientation and position. 
%\ac{@CH}

\paragraph{IMDB.}
The real-world IMDB dataset \cite{mihalkova2007} includes relations between movies, actors, directors, movie genre, and gender. It has been created from the International Movie Database (IMDB.com) database. We learn the relation \emph{workedunder/2}, a more complex variant \emph{workedwithsamegender/2}, and the disjunction of the two.

\paragraph{Chess.} The task is to learn a chess pattern in the king-rook-king (\emph{krk}) endgame, which is the chess ending with white having a king and a rook and black having a king. We learn the concept of white rook protection by the white king \cite{celine:bottom}.

\paragraph{IGGP.}
In \emph{inductive general game playing} \cite{iggp} (IGGP) the task is to induce a hypothesis to explain game traces from the general game playing competition \cite{ggp}.
Although seemingly a toy problem, IGGP is representative of many real-world problems, such as inducing semantics of programming languages \cite{DBLP:conf/ilp/BarthaC19}. 
We use six games: \emph{minimal decay (md)}, \emph{rock, paper, scissors (rps)}, \emph{buttons}, \emph{attrition}, \emph{centipede}, and \emph{coins}.
% We learn the next relation in each game. %Examples of solutions for these games are shown in Listings \ref{minimaldecay}, \ref{buttons}, \ref{rps}, \ref{coins}.

\paragraph{Program Synthesis.} This dataset includes list transformation tasks. It involves learning recursive programs which has been identified as a difficult challenge for ILP systems \cite{ilp20}. 

\subsection{Experimental Setup}
We measure the mean and standard error of the predictive accuracy and learning time.
%We enforce a timeout of 600s per task.
%We repeat all the experiments 20 times.
%We measure the mean and standard error.
We use a 3.8 GHz 8-Core Intel Core i7 with 32GB of ram.
The systems use a single CPU.

\textbf{Q1.}
We compare the performance of \popper{} and \name{} on all the tasks.
We use \popper{} 2.0.0 \cite{popper2}.

\textbf{Q2.}
We compare the performance of \popper{} and \name{} when varying the size of the hypothesis space.
We vary the maximum size of a rule allowed in a hypothesis, ie the maximum number of literals in a rule.
We focus on the \emph{md} task.

\subsection{Experimental Results}

\subsubsection{Comparison against other ILP systems}

We compare\footnote{
We also tried to use \ilasp{} \cite{ilasp}.
However, \ilasp{} first pre-computes every possible rule in a hypothesis space.
This approach is infeasible for our datasets.
For instance, on the trains tasks, \ilasp{} took 2 seconds to pre-compute rules with three body literals; 20 seconds for rules with four body literals; and 12 minutes for rules with five body literals.  
Since the simplest train task requires rules with six body literals, \ilasp{} is unusable.
In addition, \ilasp{} cannot learn Prolog programs so is unusable in the synthesis tasks.
} \name{} against \popper{}, \metagol{} \cite{metagol}, and \ale{} \cite{aleph}:

% \paragraph{\popper{}.}
% One key experimental question is to see whether \dcc{} can 
\begin{description}
% \item[\popper{}] 
% We compare against \popper{} as \name{} directly improves on it.
% For instance, if given only a single positive example, \name{} is identical to \popper{}.
% In all the experiments, we give \name{} and \popper{} identical biases.
\item[\metagol{}] \metagol{} is one of the few systems that can learn recursive Prolog programs.
\metagol{} uses user-provided \emph{metarules} (program templates) to guide the search for a solution.
% When given suitable metarules,
We use the approximate universal set of metarules described by  \citet{reduce}.
\item[\ale{}] \ale{} excels at learning many large non-recursive rules and should excel at the trains and IGGP tasks.
Although \ale{} can learn recursive programs, it struggles to do so.
\name{} and \ale{} use similar biases so the comparison can be considered reasonably fair.
\end{description}

\paragraph{Results.}
Tables \ref{tab:q1accs} and \ref{tab:q1times} shows accuracies and learning times respectively.
% Figure \ref{fig:sols} shows example solutions.

\begin{table}[ht!]
\footnotesize
\centering
\begin{tabular}{@{}l|cccc@{}}
% \toprule
\textbf{Task} & \textbf{\popper{}} & \textbf{\name{}} & \textbf{\ale{}} & \textbf{\metagol{}}\\
\midrule
\emph{trains1} & \textbf{100} $\pm$ 0 & \textbf{100} $\pm$ 0 & \textbf{100} $\pm$ 0 & 27 $\pm$ 0 \\
\emph{trains2} & 98 $\pm$ 0 & 98 $\pm$ 0 & \textbf{99} $\pm$ 0 & 19 $\pm$ 0 \\
\emph{trains3} & 99 $\pm$ 0 & 99 $\pm$ 0 & \textbf{100} $\pm$ 0 & 79 $\pm$ 0 \\
\emph{trains4} & \textbf{100} $\pm$ 0 & \textbf{100} $\pm$ 0 & \textbf{100} $\pm$ 0 & 32 $\pm$ 0 \\

\midrule
\emph{zendo1} & \textbf{99} $\pm$ 0 & \textbf{99} $\pm$ 0 & \textbf{99} $\pm$ 0 & 69 $\pm$ 7 \\
\emph{zendo2} & 96 $\pm$ 1 & 97 $\pm$ 1 & \textbf{100} $\pm$ 0 & 50 $\pm$ 0 \\
\emph{zendo3} & 93 $\pm$ 1 & 93 $\pm$ 2 & \textbf{98} $\pm$ 0 & 50 $\pm$ 0 \\
\emph{zendo4} & \textbf{97} $\pm$ 0 & \textbf{97} $\pm$ 0 & 96 $\pm$ 0 & 50 $\pm$ 0 \\
\midrule
\emph{imdb1} & \textbf{100} $\pm$ 0 & \textbf{100} $\pm$ 0 & \textbf{100} $\pm$ 0 & 16 $\pm$ 0 \\
\emph{imdb2} & \textbf{100} $\pm$ 0 & \textbf{100} $\pm$ 0 & 50 $\pm$ 0 & 50 $\pm$ 0 \\
\emph{imdb3} & \textbf{100} $\pm$ 0 & \textbf{100} $\pm$ 0 & 50 $\pm$ 0 & 50 $\pm$ 0 \\
\midrule
\emph{krk} & \textbf{99} $\pm$ 0 & \textbf{99} $\pm$ 0 & 98 $\pm$ 0 & 50 $\pm$ 0 \\
\midrule
\emph{rps} & \textbf{100} $\pm$ 0 & \textbf{100} $\pm$ 0 & 18 $\pm$ 0 & 18 $\pm$ 0 \\
\emph{centipede} & \textbf{100} $\pm$ 0 & \textbf{100} $\pm$ 0 & 25 $\pm$ 0 & 50 $\pm$ 0 \\
\emph{md} & \textbf{100} $\pm$ 0 & \textbf{100} $\pm$ 0 & 94 $\pm$ 0 & 11 $\pm$ 0 \\
\emph{buttons} & \textbf{100} $\pm$ 0 & \textbf{100} $\pm$ 0 & 73 $\pm$ 9 & 19 $\pm$ 0 \\
\emph{attrition} & \textbf{98} $\pm$ 0 & \textbf{98} $\pm$ 0 & 95 $\pm$ 0 & 2 $\pm$ 0 \\
\emph{coins} & \textbf{100} $\pm$ 0 & \textbf{100} $\pm$ 0 & 17 $\pm$ 0 & 17 $\pm$ 0 \\
\emph{buttons-goal} & 98 $\pm$ 1 & 99 $\pm$ 0 & \textbf{100} $\pm$ 0 & 50 $\pm$ 0 \\

\emph{coins-goal} & \textbf{100} $\pm$ 0 & \textbf{100} $\pm$ 0 & 93 $\pm$ 0 & 50 $\pm$ 0 \\

\midrule
\emph{dropk} & \textbf{100} $\pm$ 0 & \textbf{100} $\pm$ 0 & 54 $\pm$ 4 & 50 $\pm$ 0 \\
\emph{droplast} & \textbf{100} $\pm$ 0 & \textbf{100} $\pm$ 0 & 50 $\pm$ 0 & 50 $\pm$ 0 \\
\emph{evens} & \textbf{100} $\pm$ 0 & \textbf{100} $\pm$ 0 & 58 $\pm$ 3 & 50 $\pm$ 0 \\
\emph{finddup} & 98 $\pm$ 0 & \textbf{99} $\pm$ 0 & 50 $\pm$ 0 & 50 $\pm$ 0 \\
\emph{last} & \textbf{100} $\pm$ 0 & \textbf{100} $\pm$ 0 & 50 $\pm$ 0 & 60 $\pm$ 6 \\
\emph{len} & \textbf{100} $\pm$ 0 & \textbf{100} $\pm$ 0 & 50 $\pm$ 0 & 50 $\pm$ 0 \\
\emph{sorted} & \textbf{97} $\pm$ 2 & \textbf{97} $\pm$ 2 & 71 $\pm$ 3 & 50 $\pm$ 0 \\
\emph{sumlist} & 90 $\pm$ 6 & \textbf{100} $\pm$ 0 & 50 $\pm$ 0 & 65 $\pm$ 7 \\

% \bottomrule
\end{tabular}
\caption{
Predictive accuracies. 
We round accuracies to integer values. 
The error is standard deviation.
}
\label{tab:q1accs}
\end{table}

\begin{table}[ht!]
\footnotesize
\centering
\begin{tabular}{@{}l|cccc@{}}
% \toprule
\textbf{Task} & \textbf{\popper{}} & \textbf{\name{}} & \textbf{\ale{}} & \textbf{\metagol{}}\\
\midrule
\emph{trains1} & 5 $\pm$ 0.1 & 4 $\pm$ 0.1 & \textbf{2} $\pm$ 0.3 & \emph{timeout} \\
\emph{trains2} & 5 $\pm$ 0.2 & 4 $\pm$ 0.3 & \textbf{1} $\pm$ 0.1 & \emph{timeout} \\
\emph{trains3} & 27 $\pm$ 0.8 & 22 $\pm$ 0.6 & \textbf{4} $\pm$ 0.6 & \emph{timeout} \\
\emph{trains4} & 24 $\pm$ 0.8 & 20 $\pm$ 0.5 & \textbf{13} $\pm$ 1 & \emph{timeout} \\
\midrule
\emph{zendo1} & 8 $\pm$ 2 & 6 $\pm$ 1 & \textbf{0.6} $\pm$ 0.1 & 725 $\pm$ 193 \\
\emph{zendo2} & 32 $\pm$ 2 & 31 $\pm$ 2 & \textbf{2} $\pm$ 0.3 & \emph{timeout} \\
\emph{zendo3} & 33 $\pm$ 2 & 31 $\pm$ 1 & \textbf{3} $\pm$ 0.5 & \emph{timeout} \\
\emph{zendo4} & 24 $\pm$ 3 & 24 $\pm$ 3 & \textbf{2} $\pm$ 0.5 & \emph{timeout} \\
\midrule
\emph{imdb1} & \textbf{1} $\pm$ 0 & \textbf{1} $\pm$ 0 & 77 $\pm$ 20 & \emph{timeout} \\
\emph{imdb2} & \textbf{2} $\pm$ 0.1 & \textbf{2} $\pm$ 0 & \emph{timeout} & \emph{timeout} \\
\emph{imdb3} & 366 $\pm$ 23 & \textbf{287} $\pm$ 17 & \emph{timeout} & \emph{timeout} \\
\midrule
\emph{krk} & 48 $\pm$ 6 & 9 $\pm$ 0.6 & \textbf{0.9} $\pm$ 0.3 & 343 $\pm$ 29 \\
\midrule
\emph{rps} & 37 $\pm$ 1 & 6 $\pm$ 0.2 & \emph{timeout} & \textbf{0.1} $\pm$ 0 \\
\emph{centipede} & 47 $\pm$ 2 & 9 $\pm$ 0.2 & \textbf{0.3} $\pm$ 0 & 2 $\pm$ 0 \\
\emph{md} & 142 $\pm$ 7 & 13 $\pm$ 0.4 & \textbf{11} $\pm$ 0.6 & \emph{timeout} \\
\emph{buttons} & 686 $\pm$ 109 & \textbf{25} $\pm$ 1 & 1099 $\pm$ 28 & \emph{timeout} \\
\emph{attrition} & 410 $\pm$ 20 & 57 $\pm$ 2 & 684 $\pm$ 24 & \textbf{2} $\pm$ 0 \\
\emph{coins} & 496 $\pm$ 19 & 345 $\pm$ 18 & \emph{timeout} & \textbf{0.2} $\pm$ 0 \\
\emph{buttons-goal} & 11 $\pm$ 0.2 & 5 $\pm$ 0.1 & 35 $\pm$ 1 & \textbf{0.1} $\pm$ 0 \\
\emph{coins-goal} & 122 $\pm$ 6 & 76 $\pm$ 2 & 545 $\pm$ 17 & \textbf{0.1} $\pm$ 0 \\
\midrule
\emph{dropk} & 4 $\pm$ 0.3 & 3 $\pm$ 0.2 & 7 $\pm$ 1 & \textbf{0.1} $\pm$ 0 \\
\emph{droplast} & 41 $\pm$ 3 & \textbf{23} $\pm$ 2 & 404 $\pm$ 26 & \emph{timeout} \\
\emph{evens} & 33 $\pm$ 7 & 9 $\pm$ 1 & \textbf{2} $\pm$ 0.3 & 627 $\pm$ 190 \\
\emph{finddup} & 51 $\pm$ 8 & 32 $\pm$ 4 & \textbf{1} $\pm$ 0.2 & 1199 $\pm$ 0 \\
\emph{last} & 4 $\pm$ 0.4 & 3 $\pm$ 0.2 & \textbf{1} $\pm$ 0.2 & 960 $\pm$ 159 \\
\emph{len} & 31 $\pm$ 5 & 16 $\pm$ 2 & \textbf{1} $\pm$ 0.2 & \emph{timeout} \\
\emph{sorted} & 74 $\pm$ 5 & \textbf{23} $\pm$ 1 & 120 $\pm$ 119 & 1084 $\pm$ 115 \\
\emph{sumlist} & 554 $\pm$ 122 & 320 $\pm$ 40 & \textbf{0.3} $\pm$ 0 & 840 $\pm$ 183 \\

% \bottomrule
\end{tabular}
\caption{
Learning times.
We round times over one second to the nearest second.
The error is standard deviation.
The timeout is 20 minutes (1200s).
}
\label{tab:q1times}
\end{table}

\begin{figure*}
\centering
\footnotesize
\begin{lstlisting}[caption=trains2\label{trains2}]
east(A):-car(A,C),roof_open(C),load(C,B),triangle(B)
east(A):-car(A,C),car(A,B),roof_closed(B),two_wheels(C),roof_open(C).
\end{lstlisting}

\centering
\footnotesize
\begin{lstlisting}[caption=trains4\label{trains4}]
east(A):-has_car(A,D),has_load(D,B),has_load(D,C),rectangle(B),diamond(C).
east(A):-has_car(A,B),has_load(B,C),hexagon(C),roof_open(B),three_load(C).
east(A):-has_car(A,E),has_car(A,D),has_load(D,C),triangle(C),has_load(E,B),hexagon(B).
east(A):-has_car(A,C),roof_open(C),has_car(A,B),roof_flat(B),short(C),long(B).
\end{lstlisting}
\centering

\begin{lstlisting}[caption=zendo1\label{zendo1}]
zendo1(A):- piece(A,C),size(C,B),blue(C),small(B),contact(C,D),red(D).
\end{lstlisting}
\centering

\begin{lstlisting}[caption=zendo2\label{zendo2}]
zendo2(A):- piece(A,B),piece(A,D),piece(A,C),green(D),red(B),blue(C).
zendo2(A):- piece(A,D),piece(A,B),coord1(B,C),green(D),lhs(B),coord1(D,C).
\end{lstlisting}
\centering

\begin{lstlisting}[caption=zendo3\label{zendo3}]
zendo3(A):- piece(A,D),blue(D),coord1(D,B),piece(A,C),coord1(C,B),red(C).
zendo3(A):- piece(A,D),contact(D,C),rhs(D),size(C,B),large(B).
zendo3(A):- piece(A,B),upright(B),contact(B,D),blue(D),size(D,C),large(C).
\end{lstlisting}
\centering

\begin{lstlisting}[caption=zendo4\label{zendo4}]
zendo4(A):- piece(A,C),contact(C,B),strange(B),upright(C).
zendo4(A):- piece(A,D),contact(D,C),coord2(C,B),coord2(D,B).
zendo4(A):- piece(A,D),contact(D,C),size(C,B),red(D),medium(B).
zendo4(A):- piece(A,D),blue(D),lhs(D),piece(A,C),size(C,B),small(B).\end{lstlisting}
\centering

\begin{lstlisting}[caption=minimal decay\label{minimaldecay}]
next_value(A,B):-c_player(D),c_pressButton(C),c5(B),does(A,D,C).
next_value(A,B):-c_player(C),my_true_value(A,E),does(A,C,D),my_succ(B,E),c_noop(D).
\end{lstlisting}

\begin{lstlisting}[caption=buttons \label{buttons}]
next(A,B):-c_p(B),c_c(C),does(A,D,C),my_true(A,B),my_input(D,C).
next(A,B):-my_input(C,E),c_p(D),my_true(A,D),c_b(E),does(A,C,E),c_q(B).
next(A,B):-my_input(C,D),not_my_true(A,B),does(A,C,D),c_p(B),c_a(D).
next(A,B):-c_a(C),does(A,D,C),my_true(A,B),c_q(B),my_input(D,C).
next(A,B):-my_input(C,E),c_p(B),my_true(A,D),c_b(E),does(A,C,E),c_q(D).
next(A,B):-c_c(D),my_true(A,C),c_r(B),role(E),does(A,E,D),c_q(C).
next(A,B):-my_true(A,C),my_succ(C,B).
next(A,B):-my_input(C,D),does(A,C,D),my_true(A,B),c_r(B),c_b(D).
next(A,B):-my_input(C,D),does(A,C,D),my_true(A,B),c_r(B),c_a(D).
next(A,B):-my_true(A,E),c_c(C),does(A,D,C),c_q(B),c_r(E),my_input(D,C).
\end{lstlisting}
\begin{lstlisting}[caption=rps\label{rps}]
next_score(A,B,C):-does(A,B,E),different(G,B),my_true_score(A,B,F),beats(E,D),my_succ(F,C),does(A,G,D).
next_score(A,B,C):-different(G,B),beats(D,F),my_true_score(A,E,C),does(A,G,D),does(A,E,F).
next_score(A,B,C):-my_true_score(A,B,C),does(A,B,D),does(A,E,D),different(E,B).
\end{lstlisting}

\begin{lstlisting}[caption=coins\label{coins}]
next_cell(A,B,C):-does_jump(A,E,F,D),role(E),different(B,D),my_true_cell(A,B,C),different(F,B).
next_cell(A,B,C):-my_pos(E),role(D),c_zerocoins(C),does_jump(A,D,B,E).
next_cell(A,B,C):-role(D),does_jump(A,D,E,B),c_twocoins(C),different(B,E).
next_cell(A,B,C):-does_jump(A,F,E,D),role(F),my_succ(E,B),my_true_cell(A,B,C),different(E,D).
\end{lstlisting}

\centering
\begin{lstlisting}[caption=sorted]
f(A):-tail(A,B),empty(B).
f(A):-tail(A,D),head(A,B),head(D,C),geq(C,B),f(D).
\end{lstlisting}

\caption{Example solutions.}
\label{fig:sols}
\end{figure*}

\end{appendices}

\bibliography{manuscript.bib}

\end{document}